\documentclass[lettersize,journal]{IEEEtran}
\usepackage{amsmath,amsthm,amssymb}
\usepackage{graphicx}
\usepackage{subfig}
\usepackage{multirow}
\usepackage{booktabs}
\usepackage{makecell}
\usepackage[colorlinks=true, 
linkcolor = blue, 
citecolor = green,
urlcolor = blue]{hyperref}
\usepackage{orcidlink}

\newtheorem{theorem}{Theorem}
\newtheorem{definition}{Definition}

\newcommand*{\dif}{\mathop{}\!\mathrm{d}}

\begin{document}

\title{Low-Rank Tensor Recovery via Variational Schatten-$p$ Quasi-Norm and Jacobian Regularization}

\author{
	Zhengyun Cheng\orcidlink{0000-0002-1003-4381}, 
	Ruizhe Zhang\orcidlink{0009-0008-3040-9583}, 
	Guanwen Zhang\orcidlink{0000-0001-6036-4074}, 
	Yi Xu\orcidlink{0009-0000-9900-6143}, 
	Xiangyang Ji\orcidlink{0000-0001-9542-5260},
	and Wei Zhou\orcidlink{0000-0001-9715-6957}
	\thanks{
		This work was supported by the Innovation Foundation for Doctor Dissertation of Northwestern Polytechnical University (CX2025074).
		\emph{(Corresponding author: Guanwen Zhang.)}}
	\thanks{
		Zhengyun Cheng, Ruizhe Zhang, Guanwen Zhang and Wei Zhou are with the School of Electronics and Information, Northwestern Polytechnical University, Xi'an 710129, China. (e-mail: chengzy@mail.nwpu.edu.cn; zhangruizhe@mail.nwpu.edu.cn; guanwen.zh@nwpu.edu.cn; zhouwei@nwpu.edu.cn)}
	\thanks{
		Yi Xu is with the School of Control Science and Engineering, Dalian University of Technology, Dalian 116081, China. (e-mail: yxu@dlut.edu.cn)}
	\thanks{
		Xiangyang Ji is with Tsinghua University, Beijing 100190, China. (e-mail: xyji@tsinghua.edu.cn)}
}
\maketitle

\begin{abstract}
Higher-order tensors are well-suited for representing multi-dimensional data, such as images and videos, which typically characterize low-rank structures.
Low-rank tensor decomposition has become essential in machine learning and computer vision, but existing methods like Tucker decomposition offer flexibility at the expense of interpretability.
The CANDECOMP/PARAFAC (CP) decomposition provides a natural and interpretable structure, while obtaining a sparse solutions remains challenging. 
Leveraging the rich properties of CP decomposition, we propose a CP-based low-rank tensor function parameterized by neural networks (NN) for implicit neural representation. This approach can model the tensor both on-grid and beyond grid, fully utilizing the non-linearity of NN with theoretical guarantees on excess risk bounds.
To achieve sparser CP decomposition, we introduce a variational Schatten-p quasi-norm to prune redundant rank-1 components and prove that it serves as a common upper bound for the Schatten-p quasi-norms of arbitrary unfolding matrices. 
For smoothness, we propose a regularization term based on the spectral norm of the Jacobian and Hutchinson's trace estimator. The proposed smoothness regularization is SVD-free and avoids explicit chain rule derivations. It can serve as an alternative to Total Variation (TV) regularization in image denoising tasks and is naturally applicable to implicit neural representation.
Extensive experiments on multi-dimensional data recovery tasks, including image inpainting, denoising, and point cloud upsampling, demonstrate the superiority and versatility of our method compared to state-of-the-art approaches. The code is available at \url{https://github.com/CZY-Code/CP-Pruner}.
\end{abstract}


\section{Introduction}
Recent technological advancements have led to an increase in multi-dimensional data types. High-dimensional data often concentrates near a non-linear low-dimensional manifold \cite{lei2020geometric}, suggesting that it can be effectively represented in a lower-dimensional space. 
Leveraging low-rank prior has proven crucial in various tasks, including image and video recovery \cite{liu2025gradient,liu2025hyperspectral,li2024adaptive}, point cloud completion \cite{luo2024revisiting, luo2023low}, and 3D reconstruction \cite{chen2022tensorf, gao2023strivec, zhang2024satensorf}. Higher-order tensors provide a natural framework for modeling and processing these multi-dimensional datasets. By exploiting tensor low-rank properties, we can efficiently process high-dimensional data, enhancing algorithmic efficiency across various domains.

The low-rankness of tensors has been extensively studied for data processing and representation \cite{chen2025full, lu2019tensor, giampouras2020novel, fan2021multi, fan2023euclideannorminduced}. For higher-order tensors, the rank definition is not unique and depends on the decomposition method used. According to the low-rank hypothesis, higher-order tensors can be decomposed into combinations of lower-order or lower-dimensional subtensors through various techniques. These methods include CP decomposition \cite{kolda2009tensor, ashraphijuo2017fundamental, fan2023euclideannorminduced}, Tucker decomposition \cite{sidiropoulos2017tensor}, tensor ring \cite{zhao2016tensor, he2022patch}, tensor train \cite{liu2020low, zhang2022effective}, tensor SVD \cite{lu2019tensor, qin2022low}, and tensor network structure search \cite{li2022permutation, zheng2024svdinstn}.
The most classical tensor ranks are the multilinear rank, defined by the ranks of unfolding matrices on each mode. The CP rank is defined as the minimum number of components decomposed into a rank-1 tensor. Solving low-multilinear/CP-rank recovery problems has proven effective for obtaining sparse representations of multi-dimensional data \cite{sidiropoulos2017tensor, luo2022hlrtf, luo2023low, fan2023euclideannorminduced}. 
In practice, the Schatten-p quasi-norm is a popular proxy for low-rank recovery, enabling the solutions with sparse singular values \cite{zhang2023efficient, fan2019factor, giampouras2020novel, fan2023euclideannorminduced}, and the nuclear norm is the unique Schatten-p quasi-norm that is a convex function. 
The variational forms of Schatten-p quasi-norm has been proposed through minimize the Frobenius norm of component vectors to obtain solution of low-CP-rank decomposition  \cite{fan2023euclideannorminduced}. However, the tensor-based Schatten-p quasi-norm minimization is hard in both theory and practice, limiting their application to low-rank tensor recovery \cite{friedland2018nuclear}, such as multi-noise mixed conditions or beyond-grid tensor approximations. 
To address the above issues, we draw inspiration from \cite{luo2023low, li2025deep} to construct a CP-based implicit neural network for tensor approximation. We then optimize this neural network via the Adam optimizer under variational Schatten-p quasi-norm constraints, aiming to control the sparsity of CP decomposition.

The Local smoothness is another vitally essential prior for many real-world multi-dimensional data. The TV regularization has gained prominence in tasks like image denoising and 3D reconstruction \cite{rudin1992nonlinear, li2024local, wang2023guaranteed, peng2022exact, chen2022tensorf}, emphasizing the importance of smoothness priors. However, traditional TV operators is grid dependency thus are not suitable for beyond-grid tensor, limiting their applicability on continuous data.
To address this issue, spectral norm regularization is applied layer-wise to neural networks, enforcing Lipschitz smoothness \cite{yoshida2017spectral, miyato2018spectral}. It guarantees that small input perturbations induce bounded output variations, rendering it well-suited for continuous data. However, while the spectral norm of the model's Jacobian quantifies first-order continuity, it incurs substantial computational costs due to intensive singular value decomposition (SVD) computations and the need for explicit chain rule derivations \cite{scarvelis2024nuclear}.
Despite the effectiveness of TV regularization in many applications, its limitations with continuous data highlight the need for alternatives like spectral norm regularization. To circumvent the high computational demands of the Jacobian's spectral norm, we propose a smooth regularizer via Hutchinson's trace estimator, eliminating the need for SVD calculations and chain rule derivations.


\begin{figure}[t]
	\begin{center}
		\includegraphics[width=0.95\linewidth]{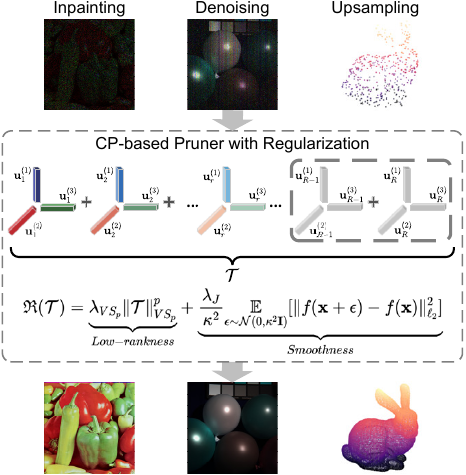}
	\end{center}
	\caption{The proposed CP-Pruner represents tensor data on or beyond meshgrids. The low-rank regularization automatically prune redundant components for a sparser CP representation. Smoothness regularization is meshgrid-independent and avoids extra TV loss for denoising task.}
	\label{fig:brief}
\end{figure}
As illustrated in Fig.~\ref{fig:brief}, our method yields the following contributions:
\begin{itemize}
	\item We propose a CP-based low-rank tensor function for implicitly representing multi-dimensional data, with theoretical guarantees on the excess risk bound.
	\item We theoretically establish that the matrix-based Schatten-$p$ quasi-norm of arbitrary tensor unfolding is upper bounded by the corresponding variational tensor-based Schatten-$p$ quasi-norm.
	\item We introduce a smooth regularizer based on the spectral norm of the Jacobian, efficiently estimated via Hutchinson's trace estimator, which eliminates the need for SVD computations and chain-rule-based derivatives.
	\item Our method applies to diverse on and beyond-grid multi-dimensional recovery tasks, including image inpainting, denoising, and point cloud upsampling. Extensive experiments show its broad applicability and superior performance over state-of-the-art methods.
\end{itemize}

\section{RELATED WORK}
\subsection{Schatten-p Quasi-Norm in Low-rank Recovery}
In most cases, the solution to matrix or tensor low-rank decomposition is not unique \cite{liu2012robust, bhaskara2014uniqueness}. To better identify the underlying low-dimensional subspace, recent studies have proposed nonconvex heuristics such as the Schatten-$p$ quasi-norm as explicit regularization\cite{zhang2023efficient}. These methods have been shown to significantly outperform their convex counterparts, i.e., $\ell_1$ and nuclear norms, in sparse/low-rank vector/matrix recovery problems across a wide range of applications \cite{shang2016tractable, shang2020unified, giampouras2020novel, fan2023euclideannorminduced}.

For instance, \cite{shang2016tractable} introduced variational definitions for $S_{1/2}$, which can be converted into the mean of the nuclear norms of two factor matrices. The generalized variational forms of the $S_p$ quasi-norm for any $p \in (0, 1)$ were derived by \cite{shang2020unified}, though these still require computing SVDs on the factor matrices, posing challenges in large-scale problems. To address this issue, \cite{fan2019factor} proposed two SVD-free variational definitions of the $S_p$ quasi-norm based on the columns of the factor matrices. In addition, \cite{fan2023euclideannorminduced} applied these variational forms naturally to CP decomposition, providing a sharper rank proxy for low-rank tensor recovery compared to the nuclear norm. 
Building on these advantages, we further introduce the variational Schatten-$p$ quasi-norm of tensors into deep learning framework, aiming to achieve a sparser CP-based neural representation and to automatically prune redundant components.

\subsection{Smoothness Regularization in Low-rank Representation}
Common methods for modeling data continuity include TV regularization and continuous basis functions. TV loss encourages uniform regions and is widely used for denoising. It has two variants: anisotropic TV, which uses the absolute distance of neighboring differences, and isotropic TV, which uses the square distance \cite{rudin1992nonlinear, cai2022anisotropic}. Its variations are tailored to data types: spatial TV for images to characterize piecewise smoothness \cite{li2024local}, spectral-spatial TV for hyperspectral images to capture spectral smoothness \cite{peng2022hyperspectral}, and temporal-spatial TV for videos \cite{yin2024spatial}.

Many studies integrate smoothness into low-rank representations by leveraging the approximate full-rank property of difference operators, ensuring gradient tensors inherit both low-rank and smoothness traits of original data \cite{wang2023guaranteed, peng2022exact}. For 3D scenes, TensoRF uses trilinear interpolation for continuous fields and TV loss to handle outliers in sparse regions \cite{chen2022tensorf}.
Alternative approaches extend CP decomposition to continuous multivariate functions via weighted smooth basis functions, e.g., Gaussians \cite{yokota2015smooth, imaizumi2017tensor}. While these embed smoothness implicitly and ensure differentiability, they are unsuitable for complex denoising tasks. 

A key limitation of TV-based methods is their difficulty in continuous data representation—NeurTV addresses this by proposing TV regularization on the neural domain, requiring explicit chain differentiation \cite{luo2025neurtv}.
In contrast, the spectral norm is well-suited for continuous data and enhances robustness to input perturbations \cite{yoshida2017spectral, miyato2018spectral}. Building on this, we propose a Jacobian spectral norm-based regularization. It avoids SVD computations and Jacobian matrix storage, achieving computational efficiency while effectively solving the challenge of applying TV loss to continuous data.

\subsection{Tensor-based Implicit Neural Representations}
Classical Implicit Neural Representations (INRs) construct differentiable functions, such as deep neural networks, to implicitly represent continuous data with respect to coordinates \cite{sitzmann2020implicit}. The implicit regularization inherent in neural networks enables the capture of nonlinear signal structures, such as manifolds, which are beyond the reach of classical linear methods like singular value decomposition and principal component analysis. Despite their significant success, INRs face challenges due to the relatively high computational cost, primarily because of the large size of the input coordinate matrix \cite{chen2022tensorf, gao2023strivec}.

To mitigate this computational burden, many studies have applied high-order tensors to represent 3D scenes \cite{chen2022tensorf, gao2023strivec, luo2023low, jin2023tensoir}. By introducing low-rank properties and low-rank decompositions, these methods transform the problem of solving large-size high-order tensors into solving a set of smaller, low-order subtensors. This approach not only reduces computational complexity but also leverages the inherent low-rank structure of the data for more effective processing.
In the field of images and video, tensor-based neural representations have also garnered significant attention \cite{yang2021implicit, luo2023low, saragadam2024deeptensor}, demonstrating promising performance in tasks such as image recovery, denoising, and super-resolution. Current tensor-based INRs primarily rely on Tucker \cite{luo2023low, fang2022bayesian, fang2024functional} or VM decompositions \cite{chen2022tensorf, gao2023strivec, jin2023tensoir}, which offer flexible structures but pose challenges in terms of interpretability.

\section{THE PROPOSED METHOD}
\begin{figure*}[t]
	\begin{center}
		\includegraphics[width=0.98\linewidth]{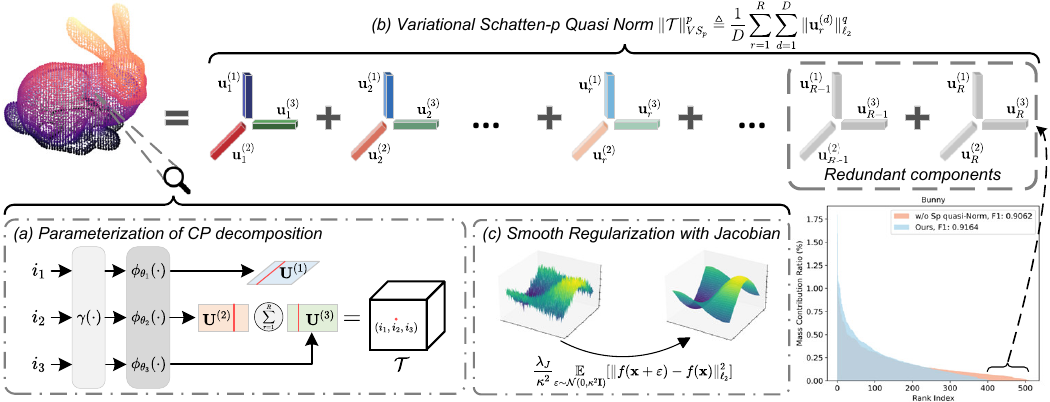}
	\end{center}
	\caption{The overview of the proposed CP-Pruner. For simplicity, we focus on the three-dimensional case, although our approach can be readily generalized to higher dimensions. (a) For holistic modeling, the variational Schatten-p quasi-norm automatically prune redundant rank-1 components for \textbf{low-rankness}, yielding a sparser CP decomposition. For detailed modeling, (b) \textbf{implicit neural networks} map spatial coordinates to sub-vectors, where Einstein summation is used to compute the tensor entry at each location; (c) The regularization based on the spectral norm of the Jacobian matrix ensures spatial \textbf{smoothness}.}
	\label{fig:overview}
\end{figure*}
\subsection{Preliminaries}
\begin{table}[t]
	\centering
	\caption{Basic matrix/tensor notation and symbols.}
	\begin{tabular}{cc}
		\toprule
		Notations &Explanations \\
		\midrule
		$t, \mathbf{t}, \mathbf{T}, \mathcal{T}$ &scalar, vector, matrix, tensor \\
		$\mathcal{T}\in\mathbb{R}^{I_1\times I_2\times\cdots\times I_D}$ &$D$th-order tensor\\
		$\mathbf{T}_{(\mathfrak{D})}$ &the unfolding matrix according to the set $\mathfrak{D}$\\
		$\mathcal{T}(i_1, \cdots, i_D)$ &the $(i_1, \cdots, i_D)$-th entry of tensor\\
		$\mathbb{R}$ &the fields of real number \\
		$rank(\cdot)$ &the rank function of matrix or tensor \\
		$\|\mathbf{T}\|_*$ &the nuclear norm of matrix\\
		$\|\mathcal{T}\|_F$ &the Frobenius norm of matrix or tensor\\
		$\|\mathcal{T}\|_\infty$ &the maximum absolute element tensor.\\
		$\|\mathcal{T}\|_{\ell_p}$ &the $\ell_p$ norm of vectorized tensor $\mathcal{T}$ \\
		$\sigma_i(\mathbf{T})$ &The $i$-th largest singular value of matrix $\mathbf{T}$ \\
		\bottomrule
	\end{tabular}
	\label{tab:notation}
\end{table}
Some frequently used notations in this paper are summarized in Table~\ref{tab:notation}. 
Let the observed tensor be denoted as $\mathcal{Y} = \mathcal{T} + \mathcal{E}$, where $\mathcal{T}$ is the true underlying tensor and $\mathcal{E}$ represents the noise tensor. Suppose we observe a few entries of $\mathcal{Y}$ randomly, with their indices forming the set $\Omega$. The goal is to recover the original tensor $\mathcal{T}$ from the incomplete observations $\mathcal{Y}$, the unconstrained version of above problem can be generally formulated as follows:
\begin{equation}
	\label{equ:lowrankrecovery}
	\min_{\mathcal{T}} \operatorname{rank}(\mathcal{T}),\quad s.t.~  \mathcal{P}_\Omega(\mathcal{T})=\mathcal{P}_\Omega(\mathcal{Y}),
\end{equation}
where the $\operatorname{rank}(\cdot)$ based on the tensor rank definitions, and the multilinear rank and CP rank are the most commonly used.
\begin{definition}
	\label{def:1}
	Let $\mathbf{u}_r^{(d)}\in\mathbb{R}^{I_d}, r\in[R], d\in[D].$ The CP rank of tensor $\mathcal{T}\in\mathbb{R}^{I_1\times I_2\cdots\times I_D}$ is defined as the minimum number of rank-one tensors that sum to $\mathcal{T}$:
	\begin{equation}
		\operatorname{rank}_{CP}(\mathcal{T})=\min\left\{R\in\mathbb{N}:\mathcal{T}=\sum_{r=1}^R\mathbf{u}_r^{(1)}\circ\cdots\circ \mathbf{u}_r^{(D)}\right\}.
	\end{equation}
	Note that $\mathbf{u}_r^{(1)}\circ\mathbf{u}_r^{(2)}\cdots\circ \mathbf{u}_r^{(D)}\in\mathbb{R}^{I_1\times I_2\times\cdots\times I_D}$ is a rank-1 tensor, and the $\circ$ is outer product. The CP decomposition equivalently,
	\begin{equation}
	\mathcal{T}(i_1,\cdots,i_D)=\sum_{r=1}^{R}\mathbf{u}^{(1)}_r(i_1)\mathbf{u}^{(2)}_r(i_2)\cdots\mathbf{u}^{(D)}_r(i_D).
	\end{equation}
	For convenience, we collect all vectors $\mathbf{u}_r^{(d)}\in\mathbb{R}^{I_d}$ on the $d$-th dimension to form the factor matrix $\mathbf{U}^{(d)}\in\mathbb{R}^{R\times I_d}\triangleq[\mathbf{u}_1^{(d)},\mathbf{u}_2^{(d)},\cdots,\mathbf{u}_R^{(d)}]^T$.
\end{definition}
\begin{definition}
	\label{def:4}
	The multilinear rank of the $D$th-order tensor $\mathcal{T}$ is mathematically defined as:
	\begin{equation}
		\operatorname{rank}_{ML}(\mathcal{T}) =\{\operatorname{rank}(\mathbf{T}_{(\{1\})}),\cdots,\operatorname{rank}(\mathbf{T}_{(\{D\})})\}.
	\end{equation}
	where the $\mathbf{T}_{(\{d\})}$ is the mode-$d$ unfolding.
\end{definition}
Minimizing the rank function directly is usually NP-hard, hence we often replace the function $\operatorname{rank}(\cdot)$ by its convex/non-convex surrogate function, such as the Schatten-$p$ quasi-norm.
\begin{definition}
	\label{def:3}
	Extending matrix Schatten-$p$ quasi-norm, the tensor Schatten-$p$ quasi-norm \cite{friedland2018nuclear,fan2023euclideannorminduced} is defined as:
	\begin{equation}
		\|\mathcal{T}\|_{S_p}=\inf\left\{\left(\sum_{r=1}^R|s_r|^p\right)^{1/p}, 0<p\leq1.\right\}
	\end{equation}
	where the $\mathcal{T}=\sum_{r=1}^R s_r\mathbf{\bar{u}}_r^{(1)}\circ\mathbf{\bar{u}}_r^{(2)}\cdots\circ\mathbf{\bar{u}}_r^{(D)}$ and the $\|\mathbf{\bar{u}}_r^{(d)}\|_{\ell_2}=1$. 
	With the tensor nuclear norm being a specific case of the tesnor Schatten-$p$ quasi-norm. 
\end{definition}

The well-known Schatten-$p$ quasi-norm serves as a better surrogate for low-rankness than the nuclear norm as $p \to 0$ \cite{shang2016tractable, shang2020unified, giampouras2020novel, nie2012low}, and it has thus been widely adopted in tensor recovery. For this context, the low-CP-rank and low-multilinear-rank recovery models \cite{fan2023euclideannorminduced, gao2020robust} are formulated as follows:
\begin{align}
	&\min_{\mathcal{T}} \|\mathcal{T}\|^p_{S_p}, \quad s.t.~  \mathcal{P}_\Omega(\mathcal{T})=\mathcal{P}_\Omega(\mathcal{Y}) \label{equ:lowCPrankrecovery},\\
	&\min_{\mathcal{T}} \frac{1}{D}\sum_{d=1}^{D}\|\mathbf{T}_{(d)}\|^p_{S_p},\quad s.t.~ \mathcal{P}_\Omega(\mathcal{T})=\mathcal{P}_\Omega(\mathcal{Y})\label{equ:lowMLRankrecovery}.
\end{align}
In Section~\ref{sec:VSpQuasiNorm}, we prove that when $\mathcal{T}$ admits a CP decomposition, these two surrogates share a common upper bound, which is the variational Schatten-$p$ quasi-norm $\|\mathcal{T}\|^p_{VS_p}$.

\subsection{Parameterization of CP-based Recovery Model}\label{sec:parameterization}
As illustrated in Fig.~\ref{fig:overview}, given that most tensors of interest for recovery in practice are approximately low-rank rather than strictly rank-$R$, we transform the low-CP-rank recovery problem from Eq.~\eqref{equ:lowrankrecovery} which has an overspecified rank $R$ into the following:
\begin{gather}
	\label{equ:Parameterization}
	\min_{\{\mathbf{W}^{(d)}_l\}}\|\mathcal{P}_\Omega(\mathcal{Y}-\mathcal{T})\|_F^2+\mathfrak{R}(\mathcal{T}),\\
	s.t.,\ 
	\mathcal{T}(i_1,i_2,\cdots,i_D)=\sum_{r=1}^{R}\mathbf{u}^{(1)}_r(i_1)\mathbf{u}^{(2)}_r(i_2)\cdots\mathbf{u}^{(D)}_r(i_D),\nonumber\\
	\mathbf{U}^{(d)}(:,i_d)= \underbrace{g(\mathbf{W}^{(d)}_L g(\mathbf{W}^{(d)}_{L-1}\cdots g(\mathbf{W}^{(d)}_1\gamma(i_d))))}_{f_d(i_d)\in\mathbb{R}^{R}}.\nonumber
\end{gather}
Here, $\{\mathbf{W}_l^{(d)}\}_{l=1,d=1}^{L,D}$ denotes the network weights for representing a $D$-order tensor, where each dimension incorporates $L$ layers of multi-layer perceptrons (MLPs). The input to the tensor function $f(\cdot)$ is $\mathbf{x} = [i_1, i_2, \cdots, i_D]$, a $D$-dimensional coordinate vector.
The tensor function $f(\mathbf{x}) = [f_1, f_2, \cdots, f_D](\mathbf{x})$ outputs the data value at any real coordinate within the input domain. Specifically, for beyond-grid data representation with infinite resolution, $\mathbf{x} \in [0, I_1] \times [0, I_2] \times \cdots \times [0, I_D]$; for on-grid data with finite resolution, $\mathbf{x} \in \{1, \cdots, I_1\} \times \{1, \cdots, I_2\} \times \cdots \times \{1, \cdots, I_D\}$.
In summary, $f$ maps a $D$-dimensional coordinate to its corresponding value, i.e., $f : \mathbb{R}^D \cup \mathbb{N}_+^D \rightarrow \mathbb{R}$, thus implicitly representing $D$-dimensional tensor data.

According to the definition~\ref{def:1} of the factor matrix $\mathbf{U}^{(d)}$ in CP decomposition, each latent function was designed to predict the factor matrix $\mathbf{U}^{(d)}$ on each demension:
\begin{equation}
	\mathbf{U}^{(d)}(:,i_d)=f_d(i_d)\triangleq(\phi_{\theta_d}\circ\gamma)(i_d).
\end{equation}
For each dimension, the latent function is composed of two functions, i.e.,
$f_d(\cdot) \triangleq (\phi_{\theta_d} \circ \gamma)(\cdot) : \mathbb{R} \rightarrow \mathbb{R}^{R}.$
Here, $\phi_{\theta_d} : \mathbb{R}^{2m} \rightarrow \mathbb{R}^{R}$ denotes a MLP with weights $\theta_d \triangleq \{\mathbf{W}_l^{(d)}\}_{l=1}^L$, and $g(\cdot)$ represents the activation function. 
To effectively learn high-frequency information and bound the Frobenius norm of the input vectors to MLPs, the Fourier feature mapping $\gamma(\cdot)$ was proposed in \cite{tancik2020fourier}. The function $\gamma(\cdot) : \mathbb{R} \rightarrow \mathbb{R}^{2m}$ maps each coordinate to the surface of a higher-dimensional hypersphere using a set of sinusoidal functions:
\begin{equation}
	\label{equ:gammafunction}
	\begin{split}
		\gamma(i_d)\triangleq&[a_1\cos(2\pi b_1 i_d),a_1\sin(2\pi b_1 i_d),\\
		&\cdots,\\
		&a_m\cos(2\pi b_m i_d),a_m\sin(2\pi b_m i_d)]^T.
	\end{split}
\end{equation}
Where the $a$ are the Fourier series coefficients, and $b$ are corresponding Fourier basis frequencies, both are hyperparameters. This transformation facilitates the learning process by providing a richer representation for the MLP to operate on.

The last term in Equ~\eqref{equ:Parameterization} comprises low-CP-rank regularization based on the variational Schatten-$p$ quasi-norm and smoothness regularization based on the Jacobian.
\begin{equation}
	\label{equ:regularizations}
	\mathfrak{R}(\mathcal{T})=\lambda_{VS_p}\|\mathcal{T}\|^p_{VS_p}+ \frac{\lambda_{J}}{\kappa^2}\mathop{\mathbb{E}}_{\epsilon\sim\mathcal{N}(0,\kappa^2\mathbf{I})}[\Vert f(\mathbf{x}+\epsilon)-f(\mathbf{x})\Vert^2_{\ell_2}].
\end{equation}
We will explain the above two regularizations in the next two subsections.

\subsection{Variational Schatten-p Quasi Norm for CP Decomposition}\label{sec:VSpQuasiNorm}
According to Definition~\ref{def:3}, we reformulate the Schatten-$p$ quasi-norm of a high-order tensor as follows:
\begin{equation}
	\label{equ:S_p-quasi-norm}
	\|\mathcal{T}\|_{S_p}^p=\inf\left\{\sum_{r=1}^R\left(\prod_{d=1}^D\|\mathbf{u}_r^{(d)}\|_{\ell_2}^q\right)^{1/D}\,\bigg|\, 0<q\leq D\right\}.
\end{equation}
Here, $q = pD$ and $|s_r| = \prod_{d=1}^D\|\mathbf{u}_r^{(d)}\|_{\ell_2}$. 
For certain choices of $p$, the factors in Eq.~\eqref{equ:S_p-quasi-norm} often involve non-smooth functions, which are not amenable to gradient-based optimization of neural network weights. To address this issue, we introduce the variational form of the Schatten-$p$ quasi-norm, derived from Eq.~\eqref{equ:S_p-quasi-norm}, into the parameterized CP-based tensor function to automatically prune redundant components when the rank $R$ is overspecified. Under general conditions, the CP decomposition of a tensor is non-unique \cite{liu2012robust, bhaskara2014uniqueness}, and our method aims to learn a sparser CP-based tensor function representation via machine learning.

Inspired by \cite{giampouras2020novel}, we theoretically establish that the tensor-based variational Schatten-$p$ quasi-norm encodes an implicit low-multilinear-rank regularization for any unfolded matrix.
\begin{theorem}
	\label{the:2}
	Follow the definition of Lemma 11 in \cite{ashraphijuo2017fundamental}, for arbitrary nonempty set $\mathfrak{D}\subset\{1,\cdots,D\}$, define $I_\mathfrak{D}\triangleq\prod_{d\in\mathfrak{D}}I_d$ and also denote $\bar{\mathfrak{D}}\triangleq\{1,\cdots,D\}\setminus\mathfrak{D}$. Let $\mathbf{T}_{(\mathfrak{D})}=\sum_{r=1}^{R}\operatorname{vec}(\circ_{d\in\mathfrak{D}}\mathbf{u}_r^{(d)})\operatorname{vec}(\circ_{d\in\bar{\mathfrak{D}}}\mathbf{u}_r^{(d)})^T\in\mathbb{R}^{I_\mathfrak{D}\times I_{\bar{\mathfrak{D}}}}$ be the unfolding of the tensor $\mathcal{T}$ corresponding to the index set $\mathfrak{D}$. 
	Consider $0<p\leq 1$, $p=q/D$, $\operatorname{rank}_{CP}(\mathcal{T})\le R$, for arbitrary nonempty set $\mathfrak{D}$, then:
	\begin{align}
		\|\mathbf{T}_{(\mathfrak{D})}\|_{S_p}^p\leq&\sum_{r=1}^R\left(\prod_{d=1}^D\|\mathbf{u}_r^{(d)}\|_{\ell_2}^q\right)^{1/D}\label{equ:objective}\\
		\leq&\frac{1}{D}\sum_{r=1}^R\sum_{d=1}^D\|\mathbf{u}_r^{(d)}\|_{\ell_2}^q.\label{equ:SumofSpNorm}
	\end{align}
	We denote the tesnor variational $S_p$ quasi-norm as $\|\mathcal{T}\|^p_{VS_p}\triangleq\frac{1}{D}\sum_{r=1}^R\sum_{d=1}^D\|\mathbf{u}_r^{(d)}\|_{\ell_2}^q$.
	The proof is provided in the supplementary material. 
\end{theorem}
It should be noted that \cite{giampouras2020novel} derived a similar inequality between the Schatten-$p$ quasi-norm and the factor vectors of a matrix, and we extend this result to higher-order tensors.
Theorem~\ref{the:2} shows that minimizing the variational Schatten-$p$ quasi-norm effectively imposes a low-rank constraint on the unfolding matrix of each mode, since the mode-$d$ unfolding $\mathbf{T}_{(\{d\})}$ is a special case of $\mathbf{T}_{(\mathfrak{D})}$ where the set $\mathfrak{D}$ contains only one element. Moreover, because $\ell_2$-norm constraints are applied to the $D$ output vectors, this formulation is well suited to the parameterized CP-based tensor function proposed in Eq.~\eqref{equ:Parameterization}, and the sparsity of the CP decomposition can be controlled by adjusting the parameter $p$.

For the choice of $p$, the formulation in Equ~\eqref{equ:SumofSpNorm} avoids non-smooth functions on the factors, whereas non-smooth functions are inherent in Definition~\ref{def:3}. Consequently, the Schatten-$p$ quasi-norm formulation in Equ~\eqref{equ:SumofSpNorm} is more tractable than that in Definition~\ref{def:3} for Adam optimizer.
It is worth noting that the objective function presented in \cite{giampouras2020novel} is a special case of Equ~\eqref{equ:SumofSpNorm}, which can be derived by applying Jensen's inequality in two dimensions case.
To the best of our knowledge, this is the first work to incorporate the variational Schatten-$p$ quasi-norm into the optimization of tensor function representations, and the automatic estimation of CP rank has been achieved through the learning of neural networks.

\subsection{Smooth Regularization with Jacobian}
Neural networks learn nonlinear functions parameterized by compositions of simpler functions. Such functions $f(\mathbf{x})$ are differentiable almost everywhere and can thus be locally approximated by linear maps specified by their Jacobian matrix $\mathbf{J}_f(\mathbf{x})$ \cite{yoshida2017spectral}. The Jacobian matrix captures the first-order partial derivatives of $f$ with respect to $\mathbf{x}$, providing a linear approximation of how $f(\mathbf{x})$ changes around the input $\mathbf{x}$.
According to Taylor's theorem, we have:
\begin{equation}
	f(\mathbf{x} + \varepsilon) = f(\mathbf{x}) + \mathbf{J}_f(\mathbf{x})\varepsilon + \mathcal{O}(\|\varepsilon\|_{\ell_2}^2),
\end{equation}
where $\varepsilon$ represents a small perturbation vector. This expansion shows that for small $\varepsilon$, the change in $f$ can be approximated by its Jacobian matrix $\mathbf{J}_f(\mathbf{x})$, with higher-order terms becoming negligible.
We consider small perturbations $\varepsilon$ as defining the neighborhood around a point $\mathbf{x}$. Reducing the spectral norm $\|\mathbf{J}_f(\mathbf{x})\|_2$ of the Jacobian matrix induced by the vector $\ell_2$ norm promotes smoothness in $f(\cdot)$. The spectral norm measures the maximum factor by which the $\ell_2$ norm of a vector can be magnified by the local linear map at $\mathbf{x}$, thus providing an upper bound on the relative change in the function's output due to these small perturbations:
\begin{equation}
	\frac{\|f(\mathbf{x} + \varepsilon) - f(\mathbf{x})\|_{\ell_2}}{\|\varepsilon\|_{\ell_2}} \leq \|\mathbf{J}_f(\mathbf{x})\|_2.
\end{equation}
In essence, controlling the spectral norm of the Jacobian helps maintain the function's smoothness with respect to input perturbations, ensuring that nearby points in the input space result in correspondingly close outputs.

Iteratively approximating the maximum singular value $\sigma_1$ via SVD is computationally expensive, especially for high-dimensional problems \cite{yoshida2017spectral}, because the cost grows exponentially with the Jacobian size, making it impractical for high-dimensional data. In contrast, the Frobenius norm avoids iteration and scales more favorably with increasing Jacobian size. Given the relationship between the Frobenius and spectral norms, we propose F-norm regularization as an alternative.
By the Schur norm inequality,
\begin{equation}
	\|\mathbf{J}_f(\mathbf{x})\|_2 \le \|\mathbf{J}_f(\mathbf{x})\|_F \le \sqrt{\min(m, n)} \,\|\mathbf{J}_f(\mathbf{x})\|_2,
\end{equation}
where $\mathbf{J}_f(\mathbf{x})$ is the Jacobian matrix of the tensor function. For the special case of $f(\mathbf{x}): \mathbb{R}^D \to \mathbb{R}$ constructed in Section~\ref{sec:parameterization}, the inequalities become equalities because the Jacobian degenerates into a vector when $\sqrt{\min(m,n)}=1$.

Computing the Jacobian matrix of a neural network with respect to its inputs is computationally and memory intensive \cite{luo2025neurtv}. Applying the chain rule for each element involves recursive derivative computations through multiple layers, and matrix operations at each layer add complexity. Storing intermediate results such as activation values and gradients vectors increases memory demand with network depth.
To reduce costs further, we use Hutchinson's trace estimator to approximate the Frobenius norm without explicitly computing $\mathbf{J}_f(\mathbf{x})$. For any matrix $\mathbf{A}$, Hutchinson's estimator states that:
\begin{equation}
	\|\mathbf{A}\|^2_F = \mathop{\mathbb{E}}_{\varepsilon\sim\mathcal{N}(0,\mathbf{I})}[\|\mathbf{A}\varepsilon\|^2_{\ell_2}].
\end{equation}
Applying this to the Jacobian matrix, we have:
\begin{equation}
	\kappa^2 \|\mathbf{J}_f(\mathbf{x})\|^2_F = \mathop{\mathbb{E}}_{\varepsilon\sim\mathcal{N}(0,\kappa^2\mathbf{I})}[\| f(\mathbf{x} + \varepsilon) - f(\mathbf{x}) \|^2_{\ell_2}] + \mathcal{O}(\kappa^2).
	\label{equ:SmoothnessReg}
\end{equation}
the above regularization can exactly capture local correlations across tensor for any direction and any order of derivatives attributed to the implicit and continuous nature of neural domain.
In particular, when $\varepsilon$ is a standard unit vector, the above equation reduces to the classical differential TV regularizer, which connects the classical TV and Jacobian-based smoothness regularization. 

By leveraging the Frobenius norm and Hutchinson's trace estimator, we balance computational efficiency and effective regularization. Our approach avoids SVD and chain-rule derivative computations and the storage of intermediate gradients, ensuring tensor function smoothness while reducing the computational burden of direct Jacobian computation and improving efficiency and performance in high-dimensional settings.

\subsection{Excess Risk Bound for CP-based Tensor Function}
Inspired by Theorem 3 in \cite{fan2021multi}, which establishes an excess risk bound for nonlinear deep tensor factorization under the Tucker decomposition, we extend the theory to our CP-based Implicit neural representation, offering insights into the performance and reliability of our method.
\begin{theorem}
	\label{the:3}
	Suppose $\hat{\mathcal{T}},\{\mathbf{W}^{(d)}_l\in\mathbb{R}^{h_{l}^{(d)}\times h_{l-1}^{(d)}}\}^{L,D}_{l=1,d=1}$ are given by Equ~\eqref{equ:Parameterization}. Suppose the weight decay $\|\mathbf{W}^{(d)}_l\|\leq\beta^{(d)}_l$, $\max(\|\mathcal{Y}\|_\infty,\|\mathcal{T}\|_\infty)\leq\xi$, and the Lipschitz constant of $g$ is $\eta$. Let $N=\prod^D_{d=1}I_d$ Then with probability at least $1-2N^{-1}$, there exists a numerical constant $c$ such that
	\begin{equation}
		\begin{aligned}
			&\frac{1}{\sqrt{N}}\|\mathcal{T}-\hat{\mathcal{T}}\|_F\\
			\leq&\frac{1}{\sqrt{|\Omega|}}\|\mathcal{P}_\Omega(\mathcal{Y}-\hat{\mathcal{T}})\|_F+\frac{1}{\sqrt{N}}\|\mathcal{E}\|_F\\
			&+
			c\xi\left(\frac{\log(\xi^{-1}\tau)\sum_{d=1}^{D}\sum_{l=0}^{L}h_{l}^{(d)} h_{l-1}^{(d)}}{|\Omega|}\right)^{1/4}\\
		\end{aligned}
		\nonumber
	\end{equation}
	where the $h_0^{(d)}=2m$, $h_{-1}^{(d)}=I_d$, $\beta_0^{(d)}=\sqrt{I_d\sum_{1}^{m}a_i^2}$ and $\tau=\eta^{LD}\prod_{d=1}^{D}\prod_{l=0}^L\beta_l^{(d)}$. The proof of the theorem is provided in the supplementary material. 
\end{theorem}
Most common activation functions, such as ReLU and sigmoid, are at worst 1-Lipschitz with respect to the $\ell_2$ norm, ensuring that the factor $\eta^{LD}$ remains reasonably bounded. Theorem~\ref{the:3} provides a probabilistic guarantee on tensor recovery accuracy by bounding the normalized Frobenius norm of the error between the true tensor $\mathcal{T}$ and its estimate $\hat{\mathcal{T}}$.
This bound covers three aspects, the discrepancy between observed values $\mathcal{Y}$ and the estimated tensor over the observed set $\Omega$, the noise in the data, and a complexity penalty that decreases as the observation set grows. The theorem ensures that, with high probability, the estimation error is controlled, providing a strong theoretical foundation for tensor completion.

\section{Experiment}
\subsection{Experimental Details}
We have conducted comprehensive comparison experiments and analysis on all the introduced tasks. Below, we first outline the important experimental settings, followed by a detailed introduction to the baselines, datasets, and results for each task. All experiments were performed on a system equipped with two Intel i7-10700K processors and one NVIDIA RTX 2080Ti GPU, and the details on hyperparameter settings are provided in the support materials.

For fair and rigorous evaluation, we carefully selected and configured experimental settings for each task. The following sections detail the metrics and methodologies for inpainting, denoising, and point cloud upsampling.
For image inpainting and denoising, we used three widely adopted metrics: Peak Signal-to-Noise Ratio (PSNR), Structural Similarity Index Measure (SSIM), and Normalized Root Mean Square Error (NRMSE), which comprehensively assess recovered image quality.
For point cloud upsampling, we employed two standard metrics, Chamfer Distance (CD) \cite{fan2017point} and F-Score \cite{knapitsch2017tanks}, that are well-suited for evaluating upsampled point cloud quality.

\subsection{Multi-Dimensional Image Inpainting}
Multi-dimensional image inpainting \cite{luo2023low,peng2022exact,wang2023guaranteed} recovers the underlying image from an observed one on a meshgrid. Given an observed image $\mathcal{Y}\in\mathbb{R}^{I_1\times I_2\times I_3}$ with observed set $\Omega\subset\Psi$, where $\Psi\triangleq\{(i_1,i_2,i_3)\mid i_1=1,\dots,I_1; i_2=1,\dots,I_2; i_3=1,\dots,I_3\}$, we formulate the problem as in Eq.~\eqref{equ:Parameterization}.
We compared our method with state-of-the-art low-rank tensor-based methods: M\textsuperscript{2}DMT \cite{fan2021multi}, LRTC-ENR \cite{fan2023euclideannorminduced} (solved by L-BFGS \cite{liu1989limited}), HLRTF \cite{luo2022hlrtf}, DeepTensor \cite{saragadam2024deeptensor}, and LRTFR \cite{luo2023low}.
We tested on color images\footnote{\url{https://sipi.usc.edu/database/database.php}}, multispectral images (MSIs) from the CAVE dataset\footnote{\url{https://www.cs.columbia.edu/CAVE/databases/multispectral/}}, and videos\footnote{\url{http://trace.eas.asu.edu/yuv/}}, under random missing with sampling rates (SR) 0.1, 0.15, 0.2, 0.25, and 0.3.

\begin{table*}[h]
	\centering
	\caption{Average quantitative results of multidimensional images by different methods. The best and second results are reported with \textbf{boldface} and \underline{underline}.}
	\label{tab:compareInpainting}
	\resizebox{1.0\linewidth}{!}{
	\begin{tabular}{cc|ccc|ccc|ccc|ccc|ccc}
		\toprule
		\multicolumn{2}{c}{Sampling rate} & \multicolumn{3}{c}{0.1} & \multicolumn{3}{c}{0.15} & \multicolumn{3}{c}{0.2} &\multicolumn{3}{c}{0.25} &\multicolumn{3}{c}{0.3} \\
		\midrule
		Data &Method &PSNR &SSIM &NRMSE &PSNR &SSIM &NRMSE &PSNR &SSIM &NRMSE &PSNR &SSIM &NRMSE &PSNR &SSIM &NRMSE\\
		\midrule
		\multirow{7}{*}{\makecell[c]{Color images\\\textit{Sailboat}\\\textit{House}\\\textit{Peppers}\\\textit{Plane}\\$(512\times 512\times 3)$}} 
		&\textit{Observed} 
		&4.846 &0.023 &0.949 &5.095 &0.030 &0.922 &5.358 &0.038 &0.895 
		&5.638 &0.046 &0.866 &5.938 &0.055 &0.837\\
		&M\textsuperscript{2}DMT\cite{fan2021multi} 
		&22.06 &0.573 &0.145 
		&23.49 &0.650 &0.137 
		&23.89 &0.692 &0.116 
		&24.45 &0.711 &0.103 
		&25.72 &0.722 &0.095\\
		&LRTC-ENR\cite{fan2023euclideannorminduced}
		&\underline{23.56} &\underline{0.628} &\underline{0.128} 
		&24.61 &0.694 &0.114 
		&25.16 &0.707 &0.102 
		&25.98 &0.737 &0.090 
		&26.87 &0.753 &0.081\\
		&HLRTF\cite{luo2022hlrtf} 
		&22.49 &0.540 &0.136 
		&24.41 &0.679 &0.110 
		&25.39 &0.711 &0.097 
		&26.34 &0.742 &0.086 
		&27.17 &0.768 &0.070\\
		&DeepTensor\cite{saragadam2024deeptensor} 
		&21.50 &0.484 &0.150 
		&24.53 &0.682 &0.118 
		&26.31 &0.717 &0.101 
		&26.14 &0.746 &0.087 
		&27.39 &0.771 &0.080 \\
		&LRTFR\cite{luo2023low} 
		&23.03 &0.597 &0.132 
		&\textbf{26.22} &\underline{0.695} &\underline{0.084} 
		&\textbf{27.49} &\underline{0.741} &\underline{0.073} 
		&\textbf{28.40} &\underline{0.769} &\textbf{0.066} 
		&\textbf{29.06} &\underline{0.790} &\textbf{0.062}\\
		&\textbf{CP-Pruner} 
		&\textbf{24.66} &\textbf{0.685} &\textbf{0.099} 
		&\underline{26.05} &\textbf{0.731} &\textbf{0.085} 
		&\underline{27.09} &\textbf{0.762} &\textbf{0.076} 
		&\underline{27.99} &\textbf{0.786} &\underline{0.069} 
		&\underline{28.73} &\textbf{0.806} &\underline{0.063}\\
		\midrule
		\multirow{7}{*}{\makecell[c]{MSIs\\\textit{Toys}\\\textit{Flowers}\\$(512\times 512\times 31)$}}
		&\textit{Observed} 
		&13.96 &0.386 &0.949 
		&14.21 &0.418 &0.922 
		&14.47 &0.447 &0.894 
		&14.75 &0.476 &0.866 
		&15.05 &0.503 &0.836\\
		&M\textsuperscript{2}DMT\cite{fan2021multi} 
		&34.89 &0.910 &0.107 &36.82 &0.928 &0.092 
		&38.19 &0.934 &0.082 &39.09 &0.950 &0.068 
		&40.37 &0.962 &0.055\\
		&LRTC-ENR\cite{fan2023euclideannorminduced}
		&35.91 &0.928 &0.094 
		&37.14 &0.935 &0.080 
		&39.33 &0.950 &0.070 
		&40.36 &0.953 &0.065 
		&40.79 &0.961 &0.053\\
		&HLRTF\cite{luo2022hlrtf} 
		&36.32 &0.935 &0.091 
		&38.64 &0.942 &0.076 
		&40.19 &0.955 &0.067 
		&41.17 &0.967 &0.051 
		&41.70 &0.977 &0.040\\
		&DeepTensor\cite{saragadam2024deeptensor} 
		&38.40 &0.947 &0.088 
		&39.99 &0.951 &0.077 
		&41.20 &0.965 &0.066 
		&42.32 &0.976 &0.043 
		&42.48 &0.986 &\underline{0.032}\\
		&LRTFR\cite{luo2023low} 
		&\underline{40.16} &\underline{0.969} &\underline{0.047} 
		&\underline{42.74} &\underline{0.982} &\underline{0.035} 
		&\underline{44.28} &\underline{0.985} &\underline{0.029} 
		&\underline{44.96} &\underline{0.987} &\underline{0.027} 
		&\underline{45.27} &\underline{0.989} &0.035\\
		&\textbf{CP-Pruner} 
		&\textbf{42.54} &\textbf{0.983} &\textbf{0.035} 
		&\textbf{45.05} &\textbf{0.989} &\textbf{0.027} 
		&\textbf{46.37} &\textbf{0.991} &\textbf{0.023} 
		&\textbf{47.14} &\textbf{0.992} &\textbf{0.021} 
		&\textbf{47.54} &\textbf{0.993} &\textbf{0.020} \\
		\midrule
		\multirow{7}{*}{\makecell[c]{Videos\\\textit{Foreman}\\\textit{Carphone}\\$(144\times 176\times 100)$}}
		&\textit{Observed} 
		&5.548 &0.017 &0.949 &5.797 &0.024 &0.922 &6.059 &0.031 &0.894 &6.340 &0.039 &0.866 &6.640 &0.046 &0.837\\
		&M\textsuperscript{2}DMT\cite{fan2021multi} 
		&23.51 &0.701 &0.124 &25.21 &0.769 &0.102 &26.47 &0.815 &0.095 &28.03 &0.840 &0.083 &28.88 &0.857 &0.078\\
		&LRTC-ENR\cite{fan2023euclideannorminduced}
		&24.23 &0.730 &0.117 &25.91 &0.793 &0.094 &27.56 &0.826 &0.088 &28.86 &0.852 &0.071 &29.73 &0.866 &0.066\\
		&HLRTF\cite{luo2022hlrtf} 
		&24.66 &0.768 &0.104 &26.49 &0.830 &0.085 &28.10 &0.837 &0.071 &29.52 &0.858 &0.063 &30.20 &0.872 &0.052\\
		&DeepTensor\cite{saragadam2024deeptensor} 
		&25.67 &0.813 &0.114 &27.34 &0.855 &0.080 &28.89 &0.851 &0.067 &29.67 &0.869 &0.058 &\underline{30.93}  &\underline{0.880} &\underline{0.050} \\
		&LRTFR\cite{luo2023low} 
		&\underline{28.53} &\underline{0.828} &\textbf{0.067} 
		&\underline{29.36} &\underline{0.854} &\underline{0.061} 
		&\underline{29.77} &\underline{0.866} &\underline{0.058} 
		&\underline{30.09} &\underline{0.873} &\underline{0.056} 
		&30.26 &0.876 &0.055\\
		&\textbf{CP-Pruner} 
		&\textbf{28.63} &\textbf{0.850} &\textbf{0.067} 
		&\textbf{30.08} &\textbf{0.874} &\textbf{0.056} 
		&\textbf{31.50} &\textbf{0.900} &\textbf{0.048} 
		&\textbf{32.51} &\textbf{0.913} &\textbf{0.043} 
		&\textbf{33.45} &\textbf{0.928} &\textbf{0.038} \\
		\bottomrule
	\end{tabular}}
\end{table*}
\begin{figure*}[h]
	\centering
	\begin{minipage}{1.\linewidth}
		\centering
		\subfloat[PSNR 5.609]
		{\includegraphics[width=0.115\linewidth]{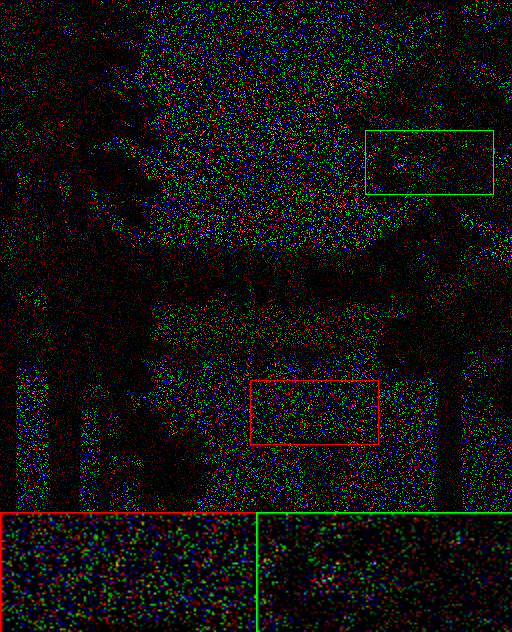}}
		\hspace{0.01cm}
		\subfloat[PSNR 19.61]
		{\includegraphics[width=0.115\linewidth]{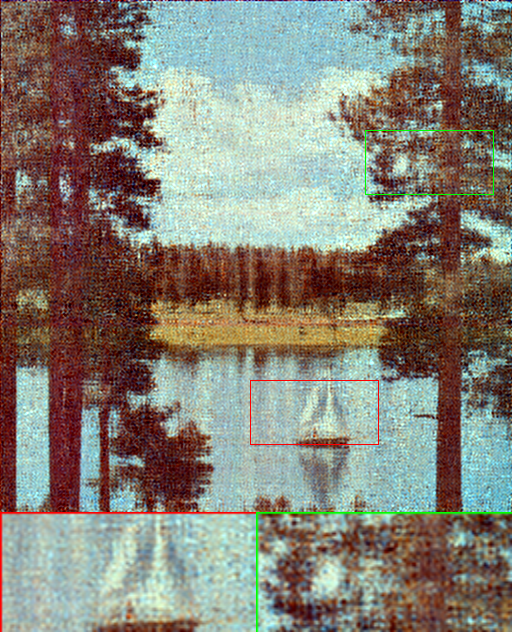}}
		\hspace{0.01cm}
		\subfloat[PSNR 20.40]
		{\includegraphics[width=0.115\linewidth]{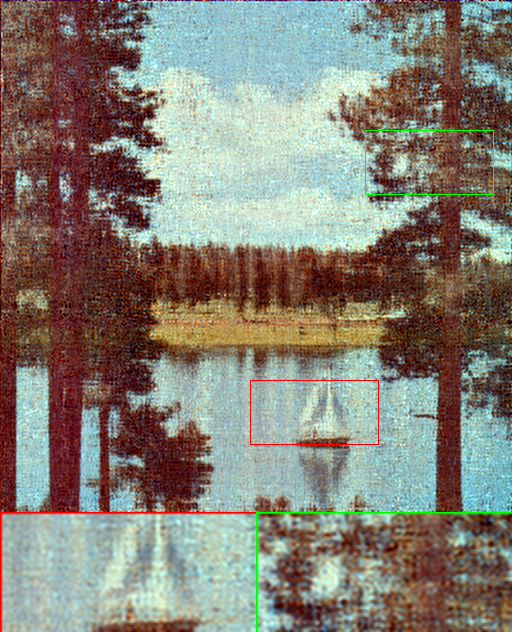}}
		\hspace{0.01cm}
		\subfloat[PSNR 21.21]
		{\includegraphics[width=0.115\linewidth]{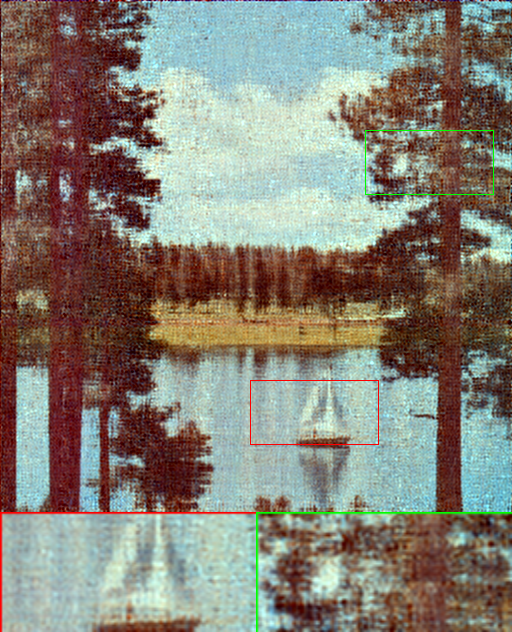}}
		\hspace{0.01cm}
		\subfloat[PSNR 21.63]
		{\includegraphics[width=0.115\linewidth]{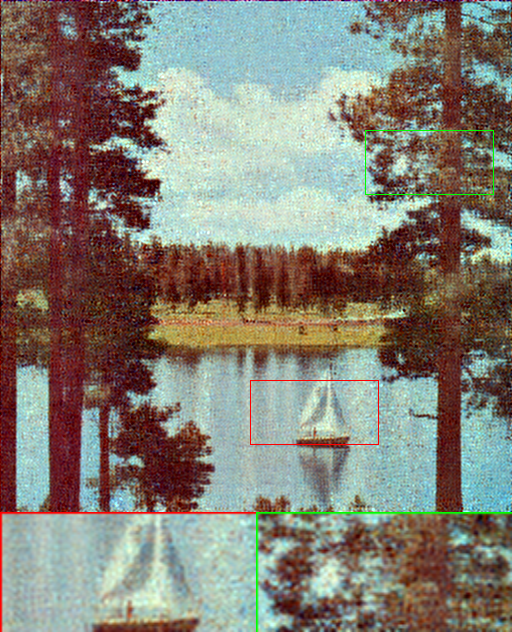}}
		\hspace{0.01cm}
		\subfloat[PSNR 21.86]
		{\includegraphics[width=0.115\linewidth]{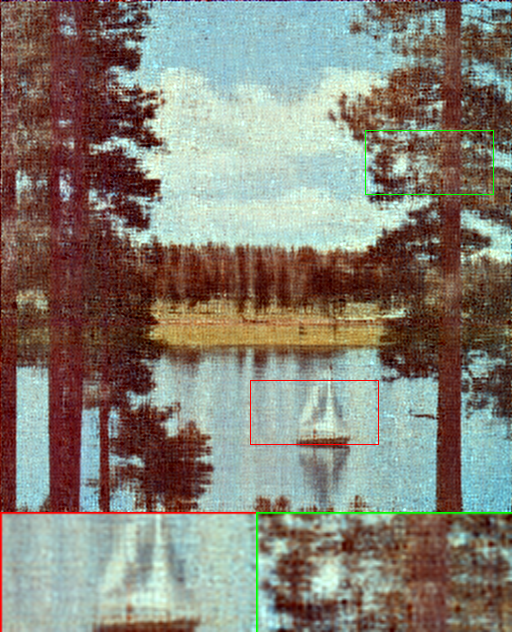}}
		\hspace{0.01cm}
		\subfloat[PSNR 22.63]
		{\includegraphics[width=0.115\linewidth]{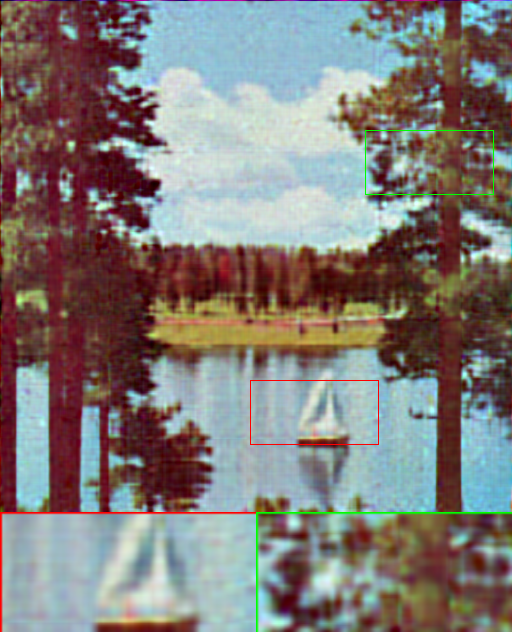}}
		\hspace{0.01cm}
		\subfloat[PSNR Inf]
		{\includegraphics[width=0.115\linewidth]{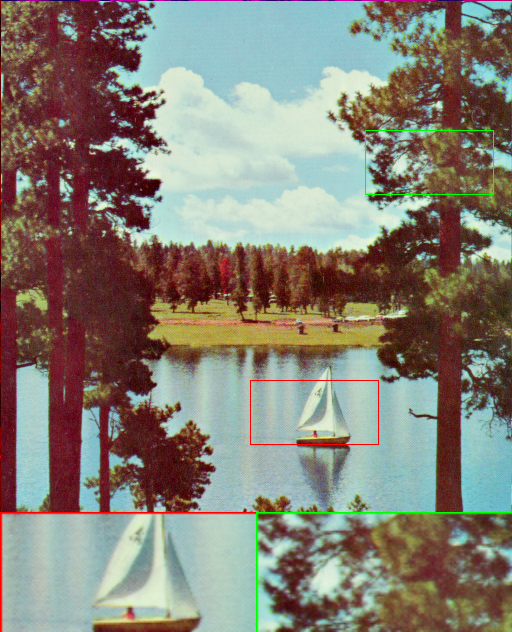}}
	\end{minipage}
	\vskip -0.3cm 
	\begin{minipage}{1.\linewidth}
		\centering
		\subfloat[PSNR 16.43]
		{\includegraphics[width=0.115\linewidth]{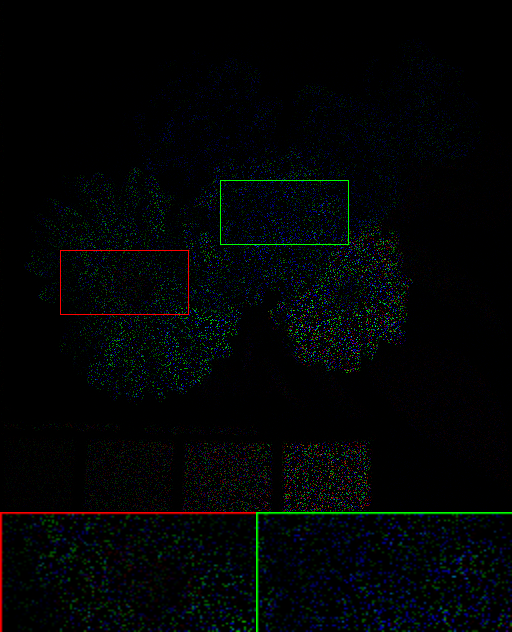}}
		\hspace{0.01cm}
		\subfloat[PSNR 38.60]
		{\includegraphics[width=0.115\linewidth]{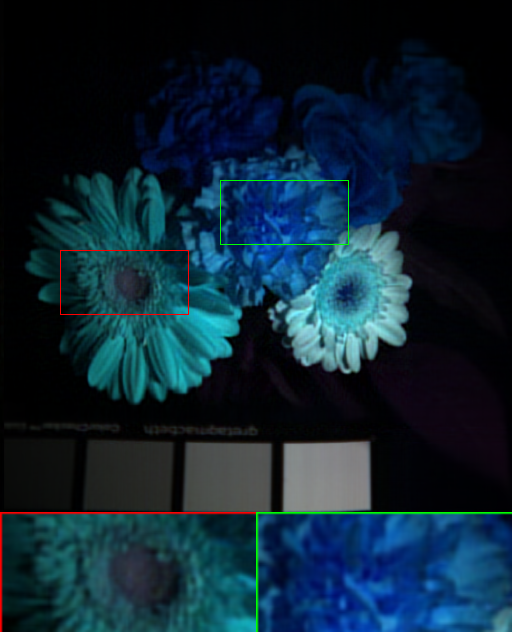}}
		\hspace{0.01cm}
		\subfloat[PSNR 39.43]
		{\includegraphics[width=0.115\linewidth]{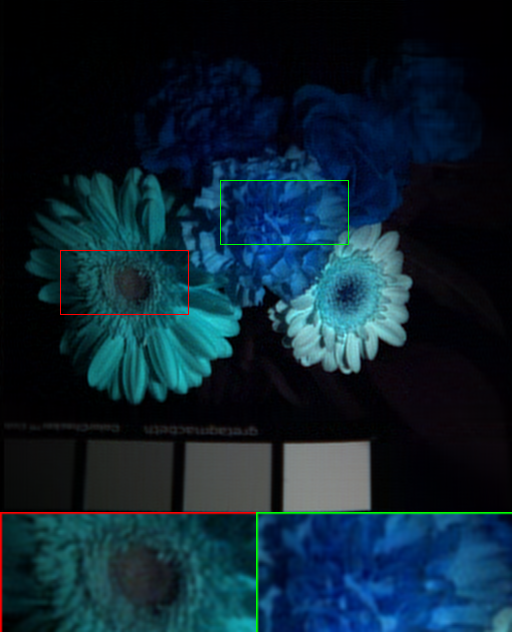}}
		\hspace{0.01cm}
		\subfloat[PSNR 41.03]
		{\includegraphics[width=0.115\linewidth]{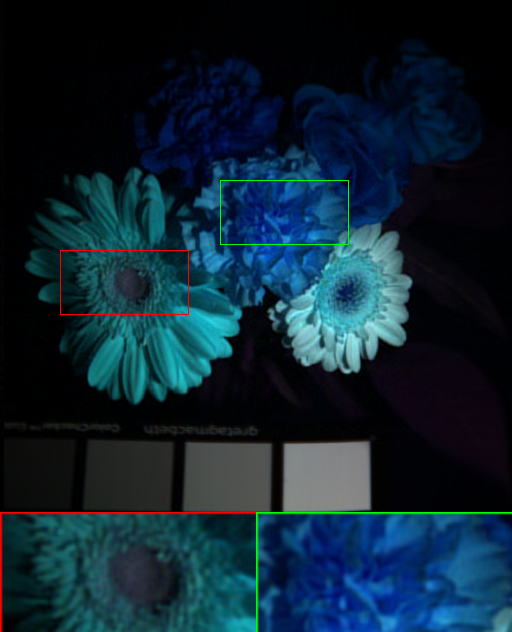}}
		\hspace{0.01cm}
		\subfloat[PSNR 42.28]
		{\includegraphics[width=0.115\linewidth]{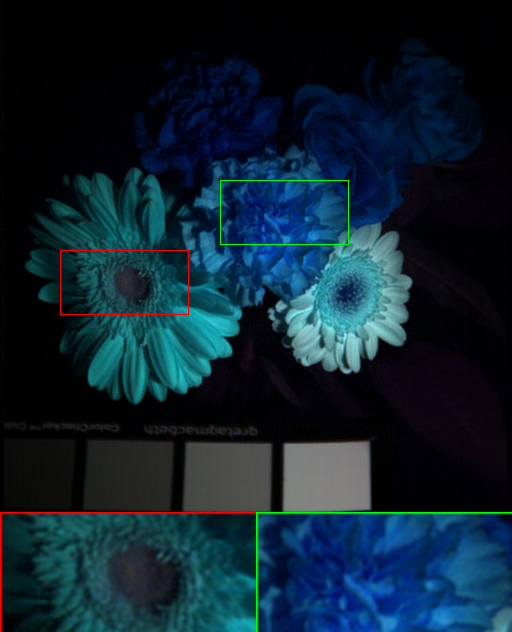}}
		\hspace{0.01cm}
		\subfloat[PSNR 43.25]
		{\includegraphics[width=0.115\linewidth]{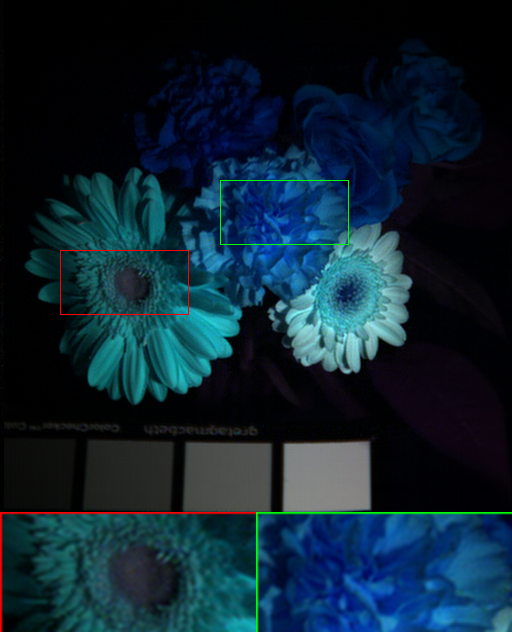}}
		\hspace{0.01cm}
		\subfloat[PSNR 44.28]
		{\includegraphics[width=0.115\linewidth]{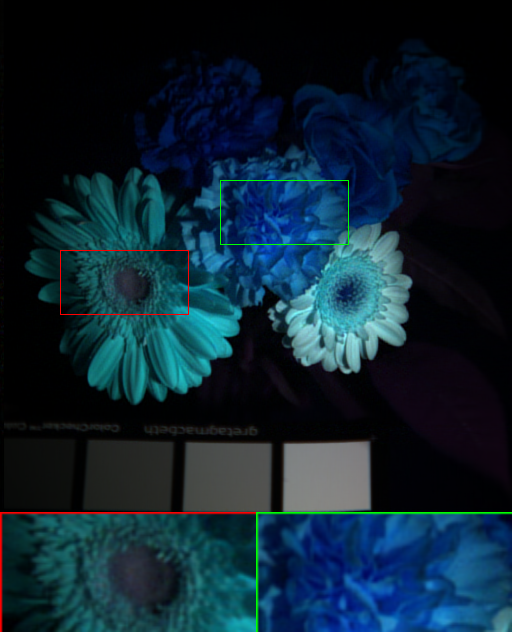}}
		\hspace{0.01cm}
		\subfloat[PSNR Inf]
		{\includegraphics[width=0.115\linewidth]{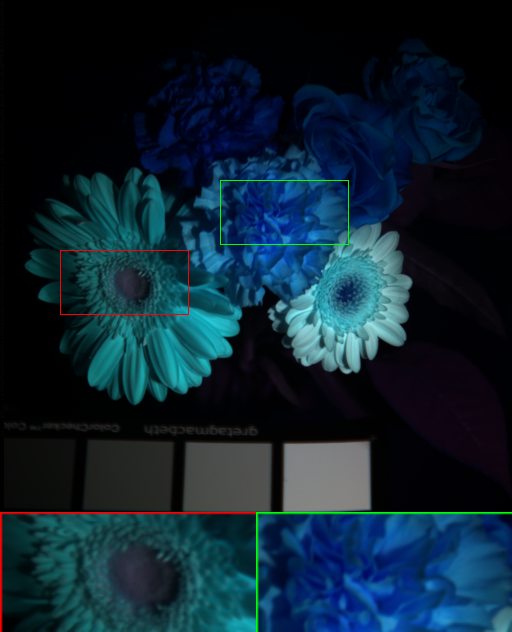}}
	\end{minipage}
	\vskip -0.3cm 
	\begin{minipage}{1.\linewidth}
		\centering
		\subfloat[\begin{tabular}{c} PSNR 7.049\\ Observed\end{tabular}]
		{\includegraphics[width=0.115\linewidth]{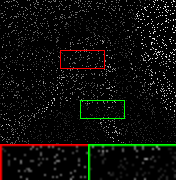}}
		\hspace{0.01cm}
		\subfloat[\begin{tabular}{c} PSNR 26.95\\ M\textsuperscript{2}DMT\cite{fan2021multi} \end{tabular}]
		{\includegraphics[width=0.115\linewidth]{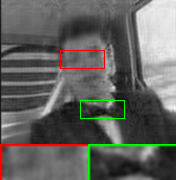}}
		\hspace{0.01cm}
		\subfloat[\begin{tabular}{c} PSNR 27.25\\ LRTC-ENR\cite{fan2023euclideannorminduced} \end{tabular}]
		{\includegraphics[width=0.115\linewidth]{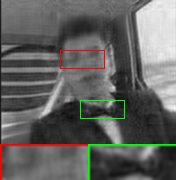}}
		\hspace{0.01cm}
		\subfloat[\begin{tabular}{c}PSNR 28.46\\ HLRTF\cite{luo2022hlrtf} \end{tabular}]
		{\includegraphics[width=0.115\linewidth]{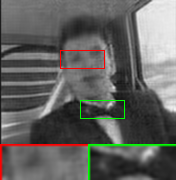}}
		\hspace{0.01cm}
		\subfloat[\begin{tabular}{c}PSNR 29.53\\ DeepTensor\cite{saragadam2024deeptensor} \end{tabular}]
		{\includegraphics[width=0.115\linewidth]{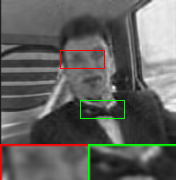}}
		\hspace{0.01cm}
		\subfloat[\begin{tabular}{c}PSNR 29.94\\ LRTFR\cite{luo2023low} \end{tabular}]
		{\includegraphics[width=0.115\linewidth]{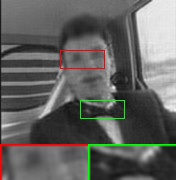}}
		\hspace{0.01cm}
		\subfloat[\begin{tabular}{c}PSNR 30.13\\ CP-Pruner \end{tabular}]
		{\includegraphics[width=0.115\linewidth]{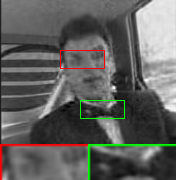}}
		\hspace{0.01cm}
		\subfloat[\begin{tabular}{c}PSNR Inf\\ Original \end{tabular}]
		{\includegraphics[width=0.115\linewidth]{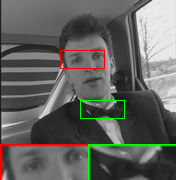}}
	\end{minipage}
	\caption{Results of multi-dimensional image inpainting by different methods on color images \textit{Sailboat}, multispectral image \textit{Flowers} and video \textit{Carphone} with SR=0.1.}
	\label{fig:Demo4Inpainting}
\end{figure*}
Quantitative and qualitative results for multi-dimensional image inpainting are presented in Table~\ref{tab:compareInpainting} and Fig.~\ref{fig:Demo4Inpainting}. Our method achieves the best performance both quantitatively and qualitatively, outperforming classical low-rank tensor representations. This strong performance stems from its ability to simultaneously encode low-rankness and smoothness.
Notably, our method consistently achieves the highest SSIM across all tasks, primarily due to the proposed Jacobian-based smoothness regularization that explicitly enforces local similarity at flexible spatial scales.

\subsection{Multispectral Image Denoising}
\begin{table*}[t]
	\centering
	\caption{The average quantitative results by different methods for multispectral image denoising.}
	\label{tab:compareDenoise}
	\resizebox{1.0\linewidth}{!}{
		\begin{tabular}{cc|ccc|ccc|ccc|ccc|ccc}
			\toprule
			\multicolumn{2}{c}{Noise case} &\multicolumn{3}{c}{Case 1} & \multicolumn{3}{c}{Case 2} & \multicolumn{3}{c}{Case 3} &\multicolumn{3}{c}{Case 4} &\multicolumn{3}{c}{Case 5} \\
			\midrule
			Data &Method &PSNR &SSIM &NRMSE &PSNR &SSIM &NRMSE &PSNR &SSIM &NRMSE &PSNR &SSIM &NRMSE &PSNR &SSIM &NRMSE\\
			\midrule
			\multirow{7}{*}{\makecell[c]{HSIs\\\textit{WDC mall}\\$(1280\times 307\times 191)$}}
			&\textit{Observed} 
			&15.44 &0.156 &0.589 &16.12 &0.201 &0.545 &15.20 &0.171 &0.606
			&15.73 &0.191 &0.570 &14.91 &0.165 &0.626\\
			&M\textsuperscript{2}DMT\cite{fan2021multi} 
			&27.23 &0.720 &0.186 
			&27.93 &0.809 &0.155 
			&26.94 &0.782 &0.145  
			&24.94 &0.751 &0.181 
			&24.35 &0.733 &0.187\\
			&LRTC-ENR\cite{fan2023euclideannorminduced}
			&28.07 &0.725 &0.157 
			&28.41 &0.820 &0.126 
			&28.79 &0.792 &0.110  
			&26.25 &0.768 &0.167 
			&25.21 &0.746 &0.151\\
			&HLRTF\cite{luo2022hlrtf} 
			&27.76 &0.736 &0.144 
			&28.95 &0.826 &0.120 
			&\underline{30.87} &\underline{0.839} &\underline{0.093}
			&28.03 &0.781 &0.151 
			&27.69 &0.742 &\underline{0.144}\\
			&DeepTensor\cite{saragadam2024deeptensor}
			&26.41 &0.741 &0.141 
			&27.88 &0.809 &0.161 
			&27.66 &0.771 &0.163  
			&25.81 &0.754 &0.158 
			&24.11 &0.734 &0.179\\
			&LRTFR\cite{luo2023low} 
			&\underline{30.89} &\underline{0.841} &\underline{0.100} 
			&\underline{33.35} &\underline{0.898} &\underline{0.075} 
			&31.08 &0.884 &0.098
			&\underline{28.11} &\underline{0.812} &\underline{0.137} 
			&\underline{27.96} &\underline{0.808} &0.140\\
			&\textbf{CP-Pruner} 
			&\textbf{31.54} &\textbf{0.873} &\textbf{0.092} 
			&\textbf{35.56} &\textbf{0.937} &\textbf{0.058} 
			&\textbf{32.83} &\textbf{0.918} &\textbf{0.080} 
			&\textbf{28.54} &\textbf{0.845} &\textbf{0.131} 
			&\textbf{28.66} &\textbf{0.838} &\textbf{0.129}\\
			\midrule
			\multirow{7}{*}{\makecell[c]{MSIs\\\textit{PaviaU}\\$(610\times 340\times 103)$}}
			&\textit{Observed} 
			&15.54 &0.150 &0.815 &16.11 &0.195 &0.763 &15.81 &0.172 &0.790
			&15.95 &0.185 &0.778 &15.67 &0.165 &0.803\\
			&M\textsuperscript{2}DMT\cite{fan2021multi} 
			&26.98 &0.711 &0.193 
			&27.72 &0.811 &0.189 
			&27.13 &0.822 &0.173
			&24.95 &0.752 &0.199 
			&25.36 &0.762 &0.233\\
			&LRTC-ENR\cite{fan2023euclideannorminduced}
			&28.58 &0.774 &0.181 
			&29.38 &0.826 &0.155 
			&28.23 &0.831 &0.169
			&26.61 &0.767 &0.183 
			&\underline{27.92} &\underline{0.808} &\underline{0.189}\\
			&HLRTF\cite{luo2022hlrtf}  
			&29.12 &0.787 &0.169 
			&30.08 &0.842 &0.132 
			&28.82 &0.838 &0.160 
			&27.89 &0.772 &0.176 
			&25.47 &0.752 &0.247\\
			&DeepTensor\cite{saragadam2024deeptensor} 
			&26.73 &0.720 &0.190 
			&27.66 &0.820 &0.179 
			&27.52 &0.820 &0.167
			&26.04 &0.788 &0.182 
			&24.83 &0.735 &0.263\\
			&LRTFR\cite{luo2023low}
			&\underline{29.52} &\underline{0.794} &\underline{0.163} 
			&\underline{31.64} &\underline{0.865} &\underline{0.128} 
			&\underline{29.84} &\underline{0.841} &\underline{0.157}
			&\underline{28.11} &\underline{0.797} &\underline{0.192} 
			&26.44 &0.762 &0.232\\
			&\textbf{CP-Pruner} 
			&\textbf{30.45} &\textbf{0.827} &\textbf{0.147} 
			&\textbf{33.69} &\textbf{0.905} &\textbf{0.101} 
			&\textbf{31.98} &\textbf{0.877} &\textbf{0.123} 
			&\textbf{29.27} &\textbf{0.846} &\textbf{0.168} 
			&\textbf{29.15} &\textbf{0.843} &\textbf{0.170}\\
			\midrule
			\multirow{7}{*}{\makecell[c]{MSIs\\\textit{Balloons}\\\textit{Beads}\\\textit{Flowers}\\\textit{Fruits}\\$(512\times 512\times 31)$}}
			&\textit{Observed} 
			&16.22 &0.084 &0.902 &16.26 &0.109 &0.900 &16.12 &0.101 &0.912 
			&16.20 &0.107 &0.906 &16.07 &0.101 &0.917\\
			&M\textsuperscript{2}DMT\cite{fan2021multi} 
			&29.80 &0.720 &0.158 
			&30.94 &0.748 &0.136 
			&30.91 &0.747 &0.170 
			&28.32 &0.731 &0.179 
			&27.17 &0.769 &0.230\\
			&LRTC-ENR\cite{fan2023euclideannorminduced}
			&31.26 &\underline{0.756} &\underline{0.152}
			&\underline{33.88} &\underline{0.845} &\underline{0.127} 
			&31.49 &0.770 &0.156 
			&28.96 &0.753 &0.174 
			&28.07 &0.780 &0.216\\
			&HLRTF\cite{luo2022hlrtf} 
			&30.57 &0.731 &0.159
			&32.75 &0.781 &0.152 
			&31.51 &0.791 &0.158 
			&29.53 &0.756 &0.173 
			&27.93 &0.774 &0.228\\
			&DeepTensor\cite{saragadam2024deeptensor} 
			&29.97 &0.725 &0.155 
			&31.03 &0.747 &0.140 
			&30.79 &0.786 &0.169 
			&29.77 &0.741 &0.179 
			&27.89 &0.772 &0.221\\
			&LRTFR\cite{luo2023low} 
			&\underline{31.32} &0.736 &0.167
			&32.89 &0.784 &0.141 
			&\underline{31.96} &\underline{0.794} &\underline{0.153} 
			&\underline{31.27} &\underline{0.776} &\underline{0.162} 
			&\textbf{29.97} &\underline{0.782} &\textbf{0.187}\\
			&\textbf{CP-Pruner}  
			&\textbf{33.63} &\textbf{0.862} &\textbf{0.130} 
			&\textbf{36.35} &\textbf{0.899} &\textbf{0.095} 
			&\textbf{33.14} &\textbf{0.889} &\textbf{0.132} 
			&\textbf{31.31} &\textbf{0.843} &\textbf{0.160} 
			&\underline{29.22} &\textbf{0.829} &\underline{0.203}\\
			\bottomrule
	\end{tabular}}
\end{table*}
\begin{figure*}[h]
	\centering
	\begin{minipage}{1.\linewidth}
		\centering
		\subfloat[PSNR 15.40]
		{\includegraphics[width=0.115\linewidth]{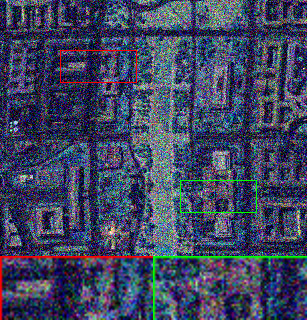}}
		\hspace{0.01cm}
		\subfloat[PSNR 27.40]
		{\includegraphics[width=0.115\linewidth]{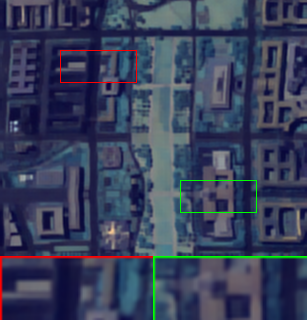}}
		\hspace{0.01cm}
		\subfloat[PSNR 27.58]
		{\includegraphics[width=0.115\linewidth]{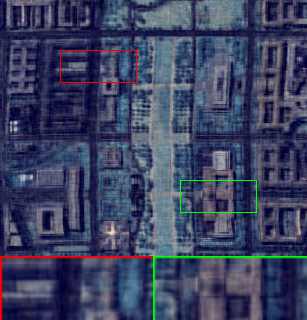}}
		\hspace{0.01cm}
		\subfloat[PSNR 27.99]
		{\includegraphics[width=0.115\linewidth]{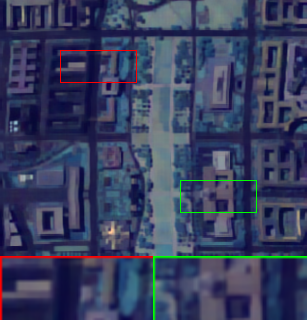}}
		\hspace{0.01cm}
		\subfloat[PSNR 28.02]
		{\includegraphics[width=0.115\linewidth]{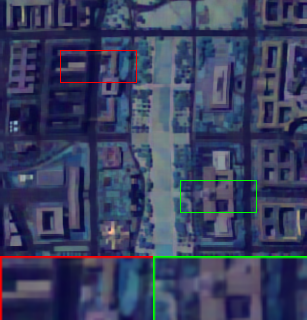}}
		\hspace{0.01cm}
		\subfloat[PSNR 30.56]
		{\includegraphics[width=0.115\linewidth]{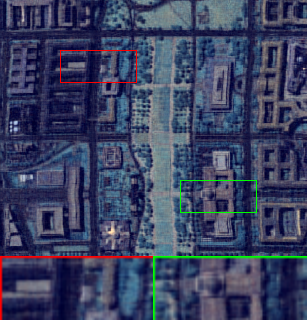}}
		\hspace{0.01cm}
		\subfloat[PSNR 31.53]
		{\includegraphics[width=0.115\linewidth]{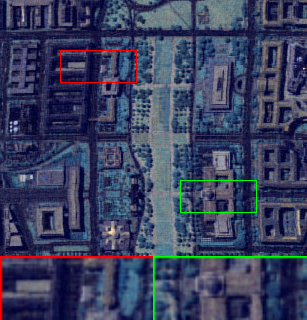}}
		\hspace{0.01cm}
		\subfloat[PSNR Inf]
		{\includegraphics[width=0.115\linewidth]{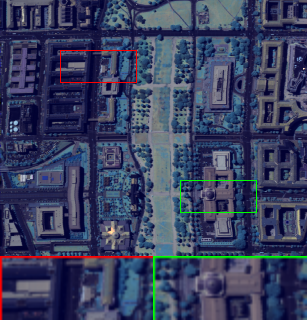}}
	\end{minipage}
	\vskip -0.3cm 
	\begin{minipage}{1.\linewidth}
		\centering
		\subfloat[PSNR 16.11]
		{\includegraphics[width=0.115\linewidth]{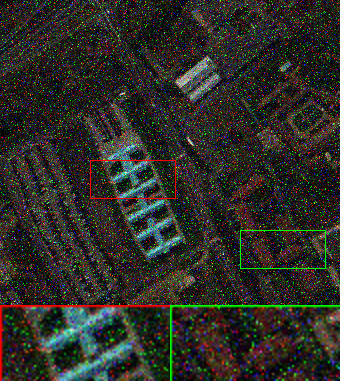}}
		\hspace{0.01cm}
		\subfloat[PSNR 27.81]
		{\includegraphics[width=0.115\linewidth]{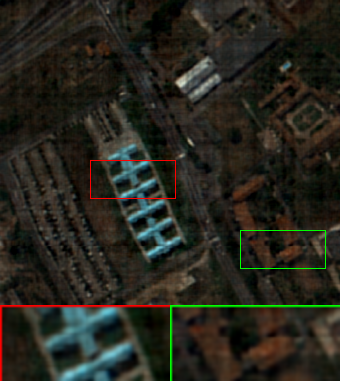}}
		\hspace{0.01cm}
		\subfloat[PSNR 29.76]
		{\includegraphics[width=0.115\linewidth]{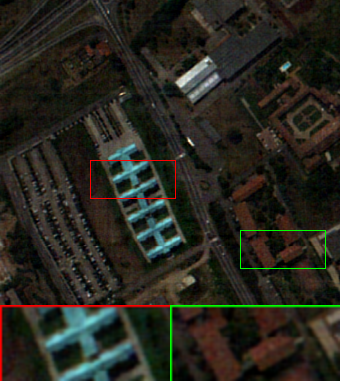}}
		\hspace{0.01cm}
		\subfloat[PSNR 30.91]
		{\includegraphics[width=0.115\linewidth]{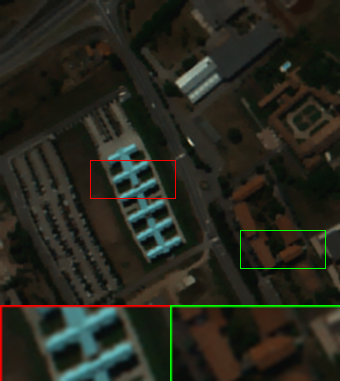}}
		\hspace{0.01cm}
		\subfloat[PSNR 27.70]
		{\includegraphics[width=0.115\linewidth]{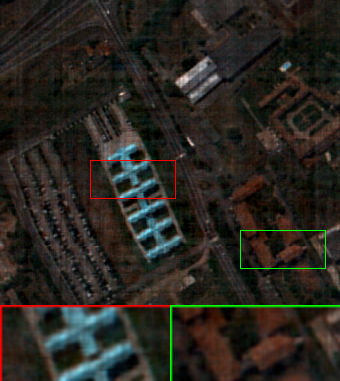}}
		\hspace{0.01cm}
		\subfloat[PSNR 31.59]
		{\includegraphics[width=0.115\linewidth]{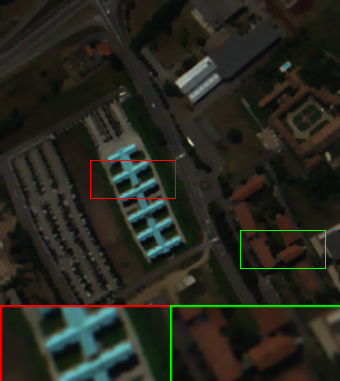}}
		\hspace{0.01cm}
		\subfloat[PSNR 33.71]
		{\includegraphics[width=0.115\linewidth]{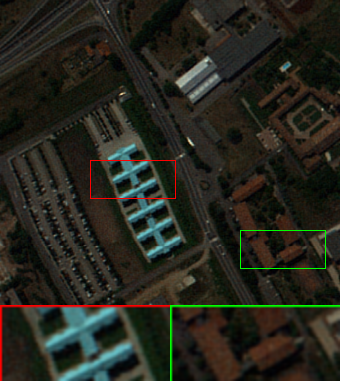}}
		\hspace{0.01cm}
		\subfloat[PSNR Inf]
		{\includegraphics[width=0.115\linewidth]{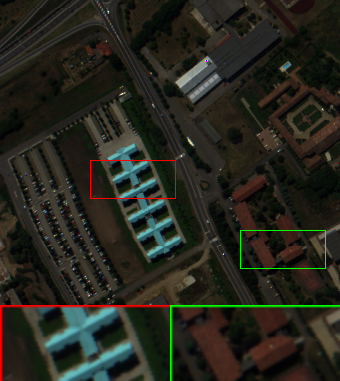}}
	\end{minipage}
	\vskip -0.3cm 
	\begin{minipage}{1.\linewidth}
		\centering
		\subfloat[PSNR 15.94]
		{\includegraphics[width=0.115\linewidth]{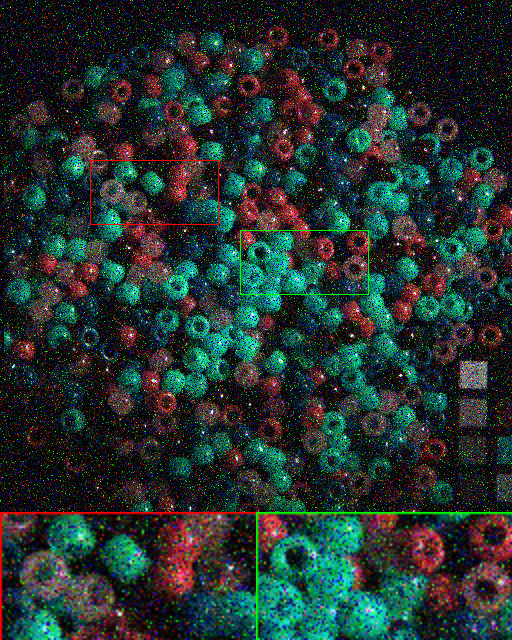}}
		\hspace{0.01cm}
		\subfloat[PSNR 26.68]
		{\includegraphics[width=0.115\linewidth]{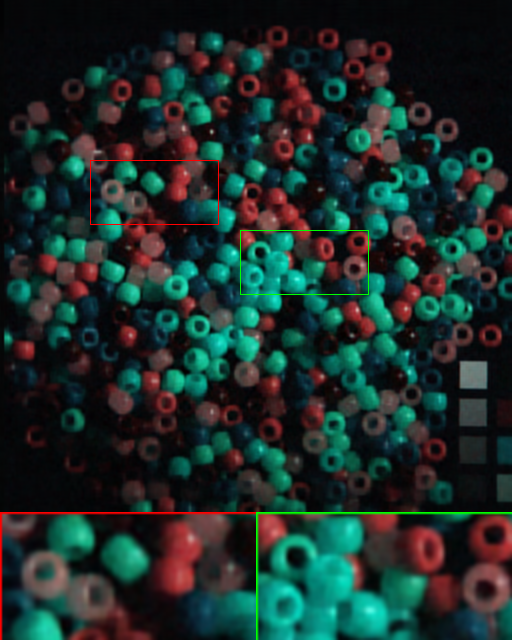}}
		\hspace{0.01cm}
		\subfloat[PSNR 27.18]
		{\includegraphics[width=0.115\linewidth]{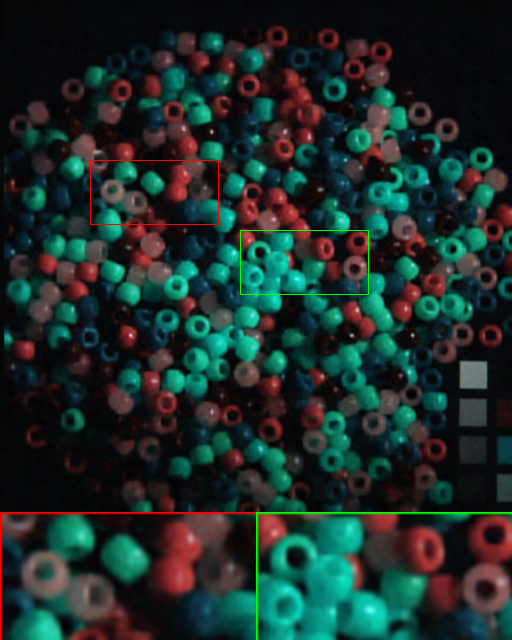}}
		\hspace{0.01cm}
		\subfloat[PSNR 27.20]
		{\includegraphics[width=0.115\linewidth]{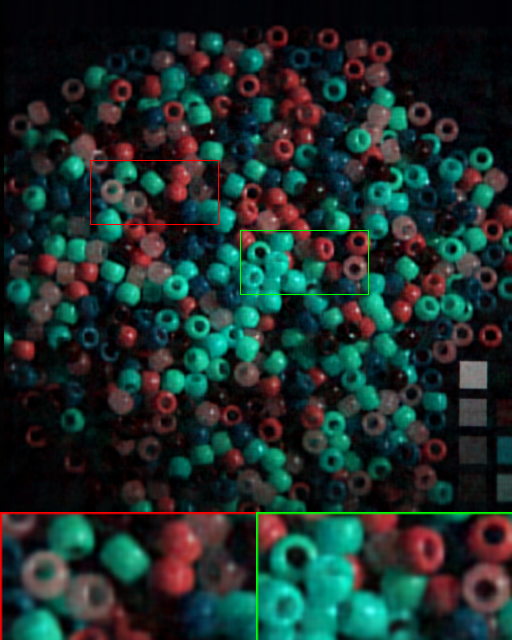}}
		\hspace{0.01cm}
		\subfloat[PSNR 26.08]
		{\includegraphics[width=0.115\linewidth]{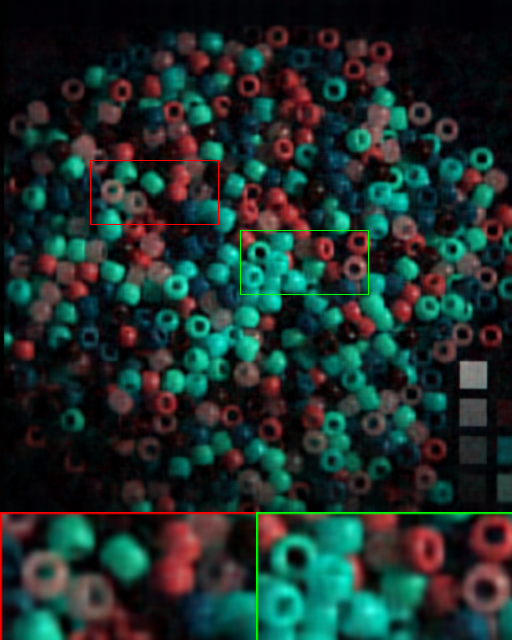}}
		\hspace{0.01cm}
		\subfloat[PSNR 27.93]
		{\includegraphics[width=0.115\linewidth]{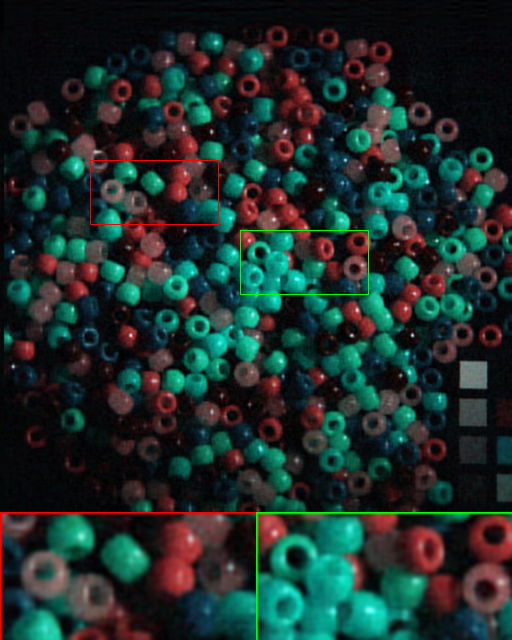}}
		\hspace{0.01cm}
		\subfloat[PSNR 28.62]
		{\includegraphics[width=0.115\linewidth]{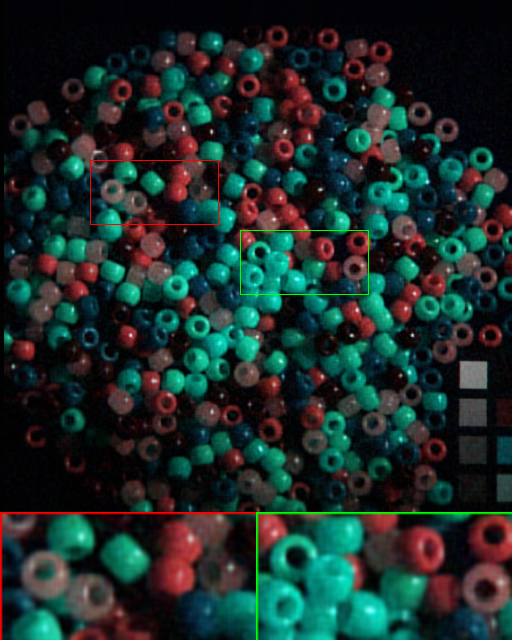}}
		\hspace{0.01cm}
		\subfloat[PSNR Inf]
		{\includegraphics[width=0.115\linewidth]{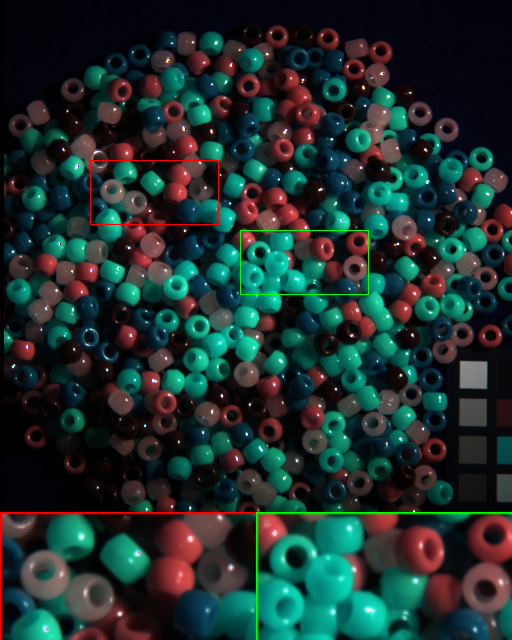}}
	\end{minipage}
	\vskip -0.3cm 
	\begin{minipage}{1.\linewidth}
		\centering
		\subfloat[PSNR 15.90]
		{\includegraphics[width=0.115\linewidth]{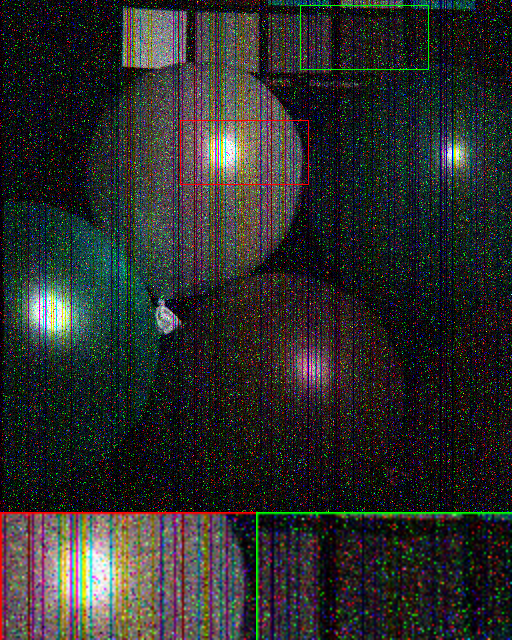}}
		\hspace{0.01cm}
		\subfloat[PSNR 29.89]
		{\includegraphics[width=0.115\linewidth]{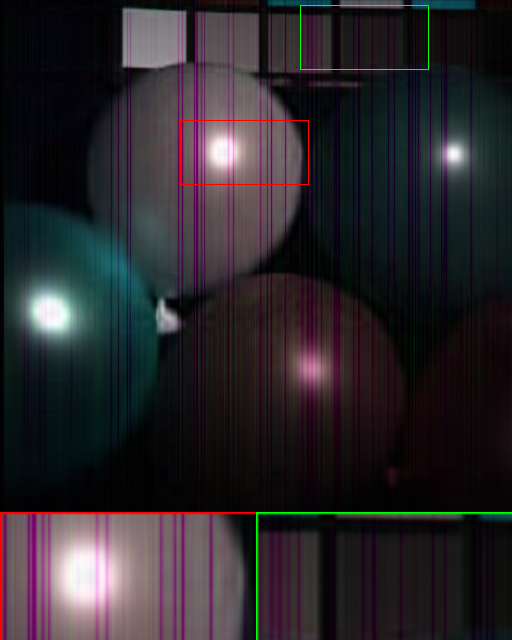}}
		\hspace{0.01cm}
		\subfloat[PSNR 32.06]
		{\includegraphics[width=0.115\linewidth]{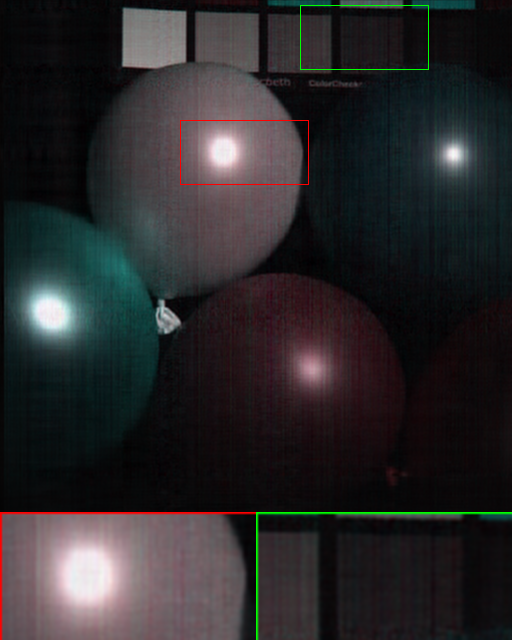}}
		\hspace{0.01cm}
		\subfloat[PSNR 32.41]
		{\includegraphics[width=0.115\linewidth]{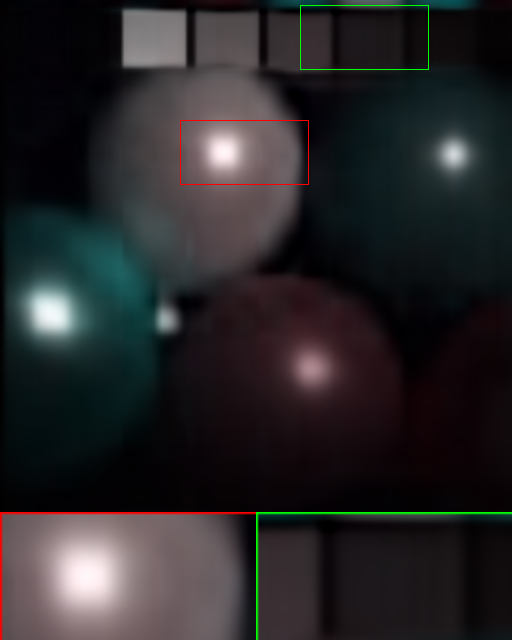}}
		\hspace{0.01cm}
		\subfloat[PSNR 30.75]
		{\includegraphics[width=0.115\linewidth]{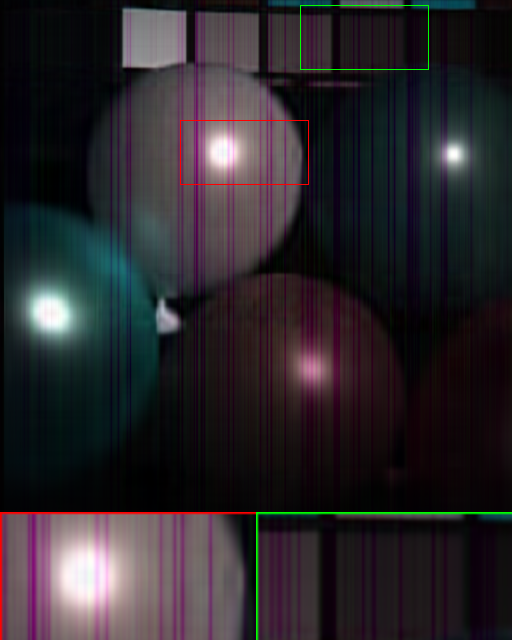}}
		\hspace{0.01cm}
		\subfloat[PSNR 33.80]
		{\includegraphics[width=0.115\linewidth]{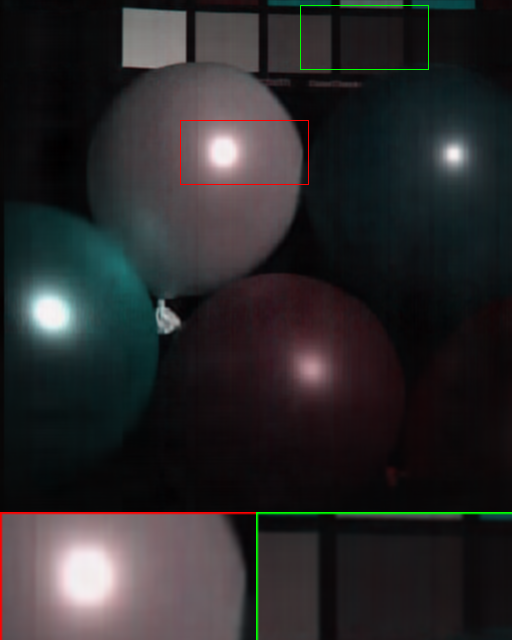}}
		\hspace{0.01cm}
		\subfloat[PSNR 34.55]
		{\includegraphics[width=0.115\linewidth]{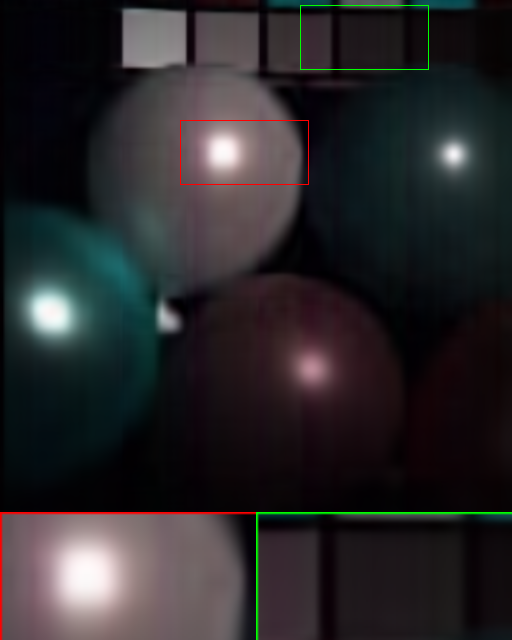}}
		\hspace{0.01cm}
		\subfloat[PSNR Inf]
		{\includegraphics[width=0.115\linewidth]{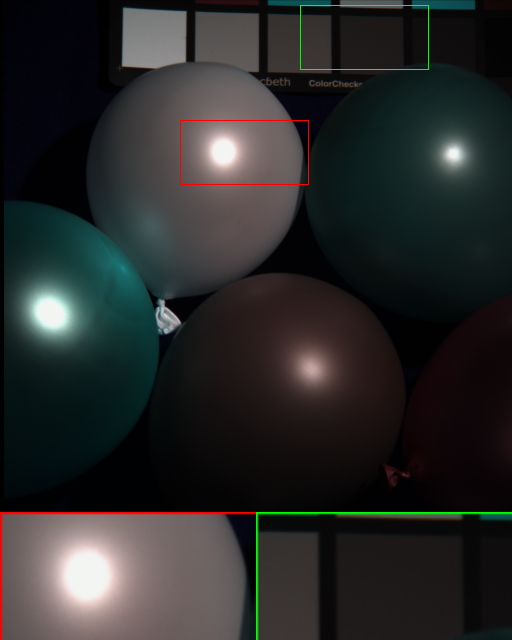}}
	\end{minipage}
	\vskip -0.3cm 
	\begin{minipage}{1.\linewidth}
		\centering
		\subfloat[\begin{tabular}{c} PSNR 16.49\\ Observed\end{tabular}]
		{\includegraphics[width=0.115\linewidth]{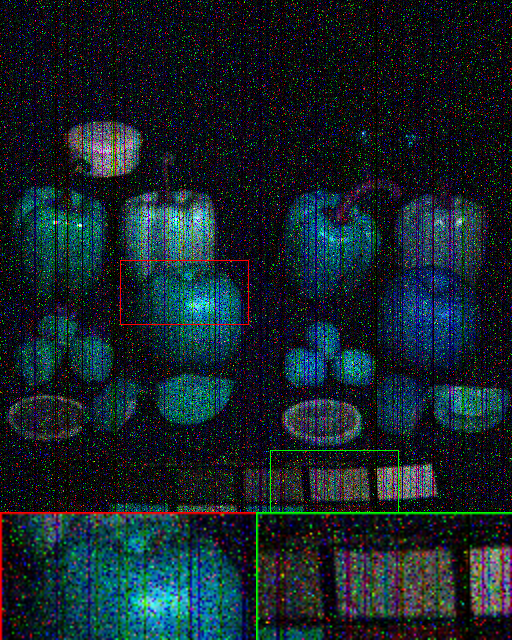}}
		\hspace{0.01cm}
		\subfloat[\begin{tabular}{c} PSNR 28.55\\ M\textsuperscript{2}DMT\cite{fan2021multi} \end{tabular}]
		{\includegraphics[width=0.115\linewidth]{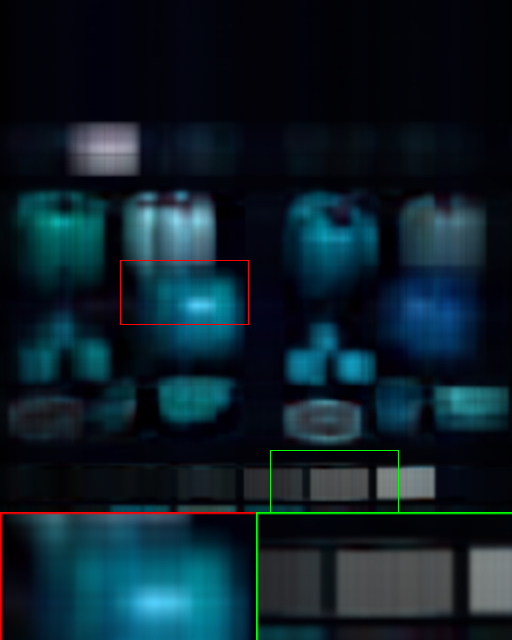}}
		\hspace{0.01cm}
		\subfloat[\begin{tabular}{c} PSNR 29.93\\ LRTC-ENR\cite{fan2023euclideannorminduced} \end{tabular}]
		{\includegraphics[width=0.115\linewidth]{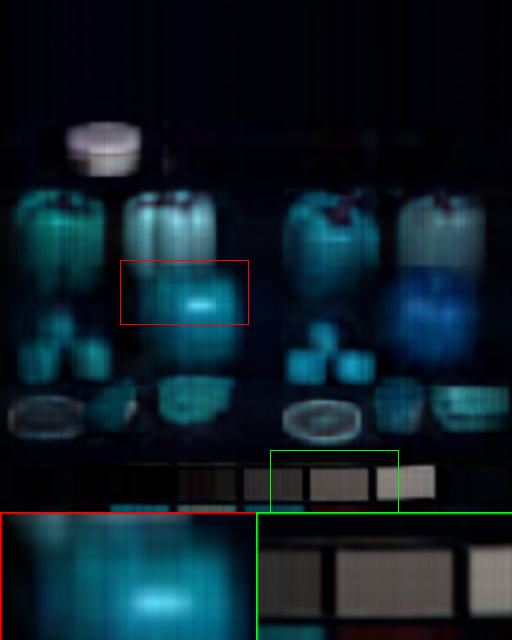}}
		\hspace{0.01cm}
		\subfloat[\begin{tabular}{c} PSNR 30.65\\ HLRTF\cite{luo2022hlrtf} \end{tabular}]
		{\includegraphics[width=0.115\linewidth]{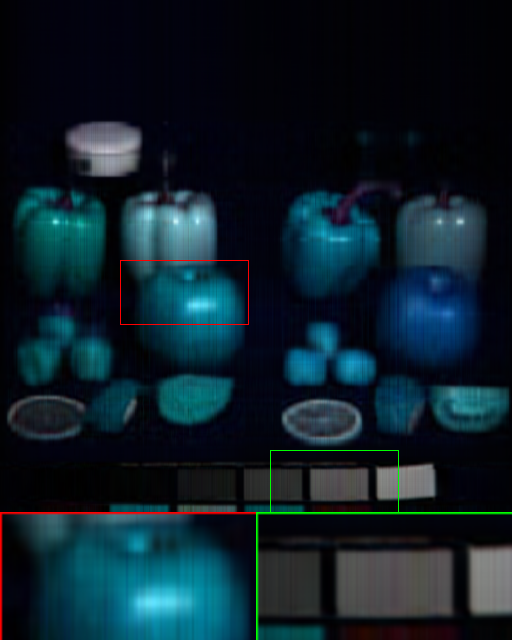}}
		\hspace{0.01cm}
		\subfloat[\begin{tabular}{c} PSNR 29.51\\ DeepTensor\cite{saragadam2024deeptensor} \end{tabular}]
		{\includegraphics[width=0.115\linewidth]{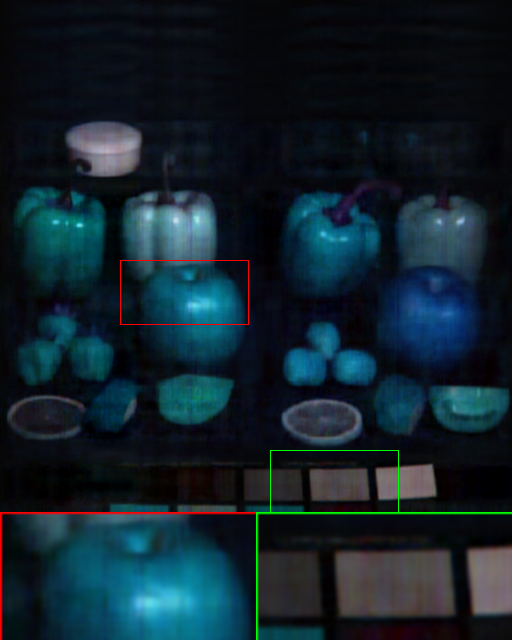}}
		\hspace{0.01cm}
		\subfloat[\begin{tabular}{c} PSNR 31.42\\ LRTFR\cite{luo2023low} \end{tabular}]
		{\includegraphics[width=0.115\linewidth]{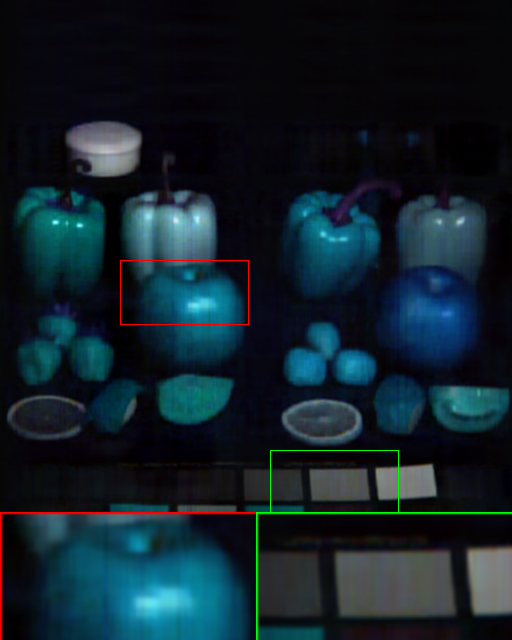}}
		\hspace{0.01cm}
		\subfloat[\begin{tabular}{c} PSNR 31.62\\ CP-Pruner \end{tabular}]
		{\includegraphics[width=0.115\linewidth]{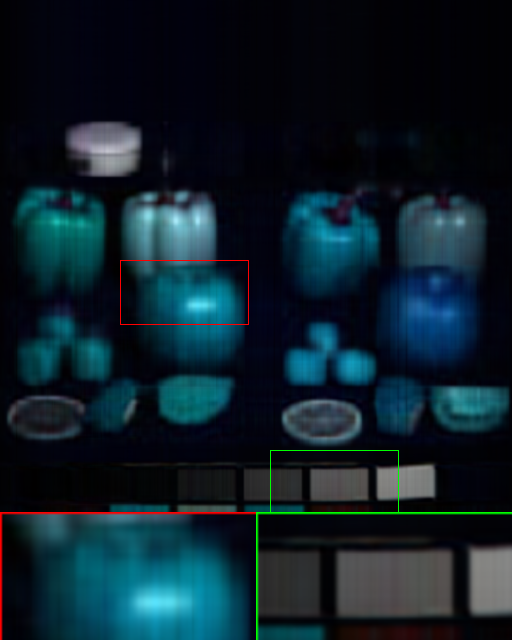}}
		\hspace{0.01cm}
		\subfloat[\begin{tabular}{c}PSNR Inf\\ Original \end{tabular}]
		{\includegraphics[width=0.115\linewidth]{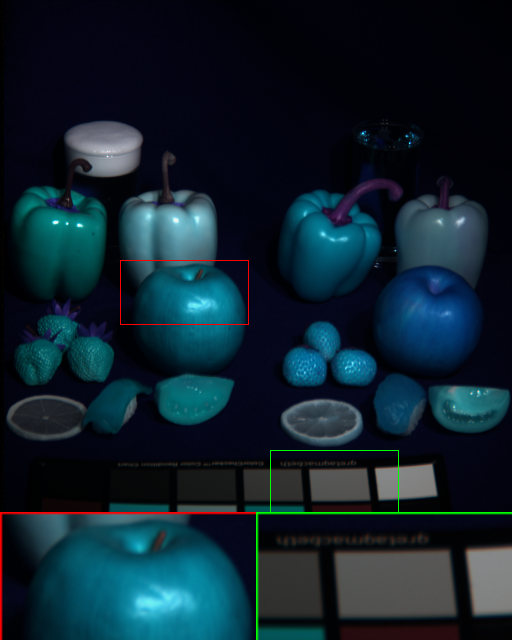}}
	\end{minipage}
	\caption{Results of multi-dimensional image denoising by different methods on HSIs \textit{WDC mall} (Case 1), \textit{PaviaU} (Case2), \textit{Beads} (Case 3), \textit{Balloons} (Case4) and \textit{Fruits} (Case5).}
	\label{fig:Demo4Denoising}
\end{figure*}
MSI denoising \cite{luo2023low, saragadam2024deeptensor, peng2022exact} recovers clean images from noisy observations on the original meshgrid; MSIs are often corrupted by mixed noise, including Gaussian, sparse, stripe noise, and dead lines. We tested four MSIs from the CAVE dataset, Pavia University, and Washington DC Mall hyperspectral images (HSIs)\footnote{\url{https://rslab.ut.ac.ir/data}} under five noise cases to evaluate algorithm robustness: Case 1 considers pure Gaussian noise with standard deviation 0.2; Case 2 combines Gaussian noise (Standard deviation=0.1) and sparse noise (Sparsity rate=0.1); Case 3 extends Case 2 by adding dead lines across all spectral bands \cite{liu2019wavelet, luo2023low}, a common real-world artifact; Case 4 extends Case 2 with 10\% stripe noise in 40\% of the bands; and Case 5 extends Case 3 with 10\% stripe noise in 40\% of the bands.

Based on Equ~\eqref{equ:Parameterization}, the optimization model of our method for multi-dimensional image denoising is formulated as
\begin{equation}
	\min_{\{\mathbf{W}^{(d)}_l\},\mathcal{S}}\|\mathcal{Y}-\mathcal{T}-\mathcal{S}\|_F^2+\lambda_S\|\mathcal{S}\|_{\ell_1}+\mathfrak{R}(\mathcal{T}),
\end{equation}
which without the commonly used TV loss and the $\mathcal{S}$ represents the sparse noise. We utilize the alternating minimization algorithm to tackle the denoising model. Specifically, we tackle the following subproblems in the $t$-th iteration:
\begin{equation}
	\begin{aligned}
		&\min_{\{\mathbf{W}^{(d)}_l\}}\|\mathcal{Y}-\mathcal{T}-\mathcal{S}^t\|_F^2+\mathfrak{R}(\mathcal{T}),\\
		&\min_{\mathcal{S}}\|\mathcal{Y}^t-\mathcal{T}^t-\mathcal{S}\|_F^2+\lambda_S\|\mathcal{S}\|_{\ell_1}.
	\end{aligned}
\end{equation}
In each iteration, we employ one step of the Adam algorithm to update the neural network weights $\{\mathbf{W}^{(d)}_l\}$. The $\mathcal{S}$ sub-problem can be exactly solved by soft-thresholding operator applied on each element of the input, i.e., $\mathcal{S}=\mathop{Soft}_{\lambda_S/2}(\mathcal{Y}^t-\mathcal{T}^t)$, where the $\mathop{Soft}_{\lambda_S/2}(\cdot)=\mathop{sgn}(\cdot)\max(|\cdot|-\frac{\lambda_S}{2}, 0)$.

We compare our method with state-of-the-art low-rank tensor-based approaches: M\textsuperscript{2}DMT \cite{fan2021multi}, LRTC-ENR \cite{fan2023euclideannorminduced} (solved via L-BFGS \cite{liu1989limited}), HLRTF \cite{luo2022hlrtf}, DeepTensor \cite{saragadam2024deeptensor}, and LRTFR \cite{luo2023low} (with additional TV loss). Results for MSI denoising are presented in Table~\ref{tab:compareDenoise} and Fig.~\ref{fig:Demo4Denoising}.  
Across the diverse noise cases tested, our method demonstrates superior performance and robustness in recovering high-quality MSIs under challenging conditions. Notably, it is the most stable among all tested algorithms, consistently performing well across different noise scenarios and datasets. Our method outperforms LRTFR \cite{luo2023low}, a Tucker-based implicit neural representation with additional TV loss, which validates the advantage of our proposed Jacobian-based smooth regularization.

\subsection{Point Cloud Upsampling}
\begin{table*}[t]
	\centering
	\caption{Quantitative results by different methods for point cloud upsampling.}
	\label{tab:comparePoint}
	\resizebox{1.0\linewidth}{!}{
		\begin{tabular}{ccccccccccccccccc}
			\toprule
			\multicolumn{1}{c}{Data} &\multicolumn{2}{c}{\textit{Table}} &\multicolumn{2}{c}{\textit{Airplane}} & \multicolumn{2}{c}{\textit{Chair}} &\multicolumn{2}{c}{\textit{Lamp}} &\multicolumn{2}{c}{\textit{Bunny}} &\multicolumn{2}{c}{\textit{Doughnut}} &\multicolumn{2}{c}{\textit{Sphere}} &\multicolumn{2}{c}{\textit{Heart}}\\
			\midrule
			Method &CD &F-Socre &CD &F-Socre &CD &F-Socre &CD &F-Socre &CD &F-Socre &CD &F-Socre &CD &F-Socre &CD &F-Socre\\
			\midrule
			\textit{Observed} 
			&0.0207 &0.3333 &0.0134 &0.3333 &0.0155 &0.3333 &0.0173 &0.3333 
			&0.0119 &0.0952 &0.0314 &0.0952 &0.0291 &0.0952 &0.0348 &0.0952\\
			SAPCU\cite{zhao2022self}
			&0.0250 &0.8695 &0.0173 &0.9029 &0.0194 &0.9347 &0.0192 &0.9138 
			&0.0227 &0.7273 &0.0686 &0.9571 &0.0278 &0.9075 &0.3010 &0.9429\\
			NeuralTPS\cite{chen2023unsupervised}
			&0.0252 &0.9051 &0.0146 &0.9188 &0.0238 &0.9045 &0.0210 &0.9475 
			&0.0020 &0.7700 &0.0462 &0.9658 &\underline{0.0209} &0.9125 &0.2163 &0.9596\\
			NeuralPoints\cite{feng2022neural}
			&0.0172 &0.9468 &0.0147 &0.9348 &0.0230 &0.8888 &0.0245 &0.8883 
			&0.0020 &0.7951 &0.0485 &0.9988 &\textbf{0.0207} &0.9367 &0.2380 &0.9571\\
			Grad-PU\cite{he2023grad}
			&0.0202 &0.9634 &\underline{0.0126} &0.9435 &0.0174 &0.9663 &0.0167 &0.9696 
			&0.0019 &0.8301 &0.0460 &0.9997 &\underline{0.0209} &0.9708 &0.2503 &0.9736\\
			LRTFR\cite{luo2023low} 
			&\underline{0.0146} &\underline{0.9858} &\underline{0.0126} &\underline{0.9482} &\underline{0.0137} &\underline{0.9805} &\underline{0.0164} &\underline{0.9860} 
			&\underline{0.0017} &\underline{0.8529} &\underline{0.0442} &\textbf{1.0000} &0.0264 &\textbf{0.9851} &\underline{0.1217} &\underline{0.9948}\\
			\textbf{CP-Pruner} 
			&\textbf{0.0113} &\textbf{0.9961} &\textbf{0.0116} &\textbf{0.9566} &\textbf{0.0112} &\textbf{0.9867} &\textbf{0.0131} &\textbf{0.9913} 
			&\textbf{0.0014} &\textbf{0.9164} &\textbf{0.0371} &\textbf{1.0000} &0.0237 &\textbf{0.9851} &\textbf{0.0985} &\textbf{0.9986}\\
			\bottomrule
	\end{tabular}}
	\label{tab:comparOfPointCloud}
\end{table*}
\begin{figure*}[t]
	\centering
	\begin{minipage}{1.\linewidth}
		\centering
		\subfloat[F1 0.3333]
		{\includegraphics[width=0.115\linewidth]{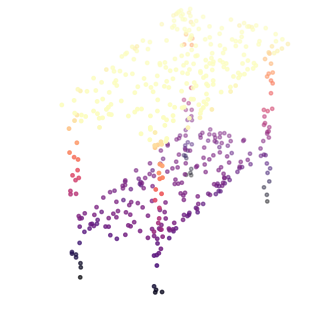}}
		\hspace{0.01cm}
		\subfloat[F1 0.8695]
		{\includegraphics[width=0.115\linewidth]{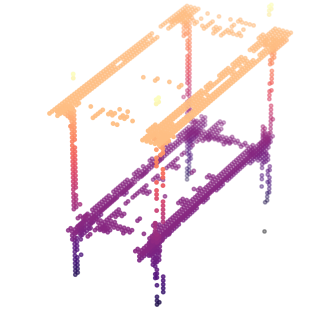}}
		\hspace{0.01cm}
		\subfloat[F1 0.9051]
		{\includegraphics[width=0.115\linewidth]{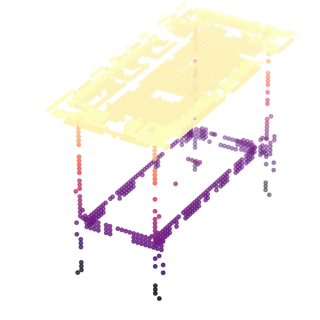}}
		\hspace{0.01cm}
		\subfloat[F1 0.9468]
		{\includegraphics[width=0.115\linewidth]{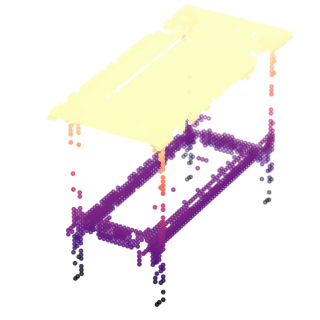}}
		\hspace{0.01cm}
		\subfloat[F1 0.9634]
		{\includegraphics[width=0.115\linewidth]{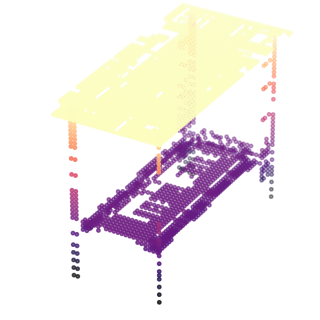}}
		\hspace{0.01cm}
		\subfloat[F1 0.9858]
		{\includegraphics[width=0.115\linewidth]{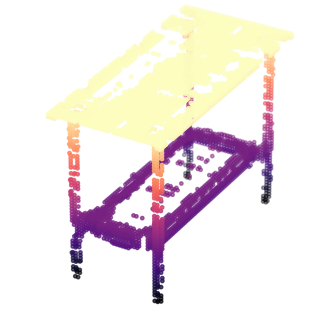}}
		\hspace{0.01cm}
		\subfloat[F1 0.9961]
		{\includegraphics[width=0.115\linewidth]{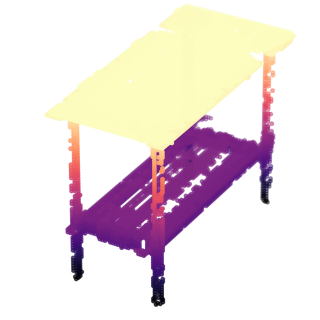}}
		\hspace{0.01cm}
		\subfloat[F1 1.0000]
		{\includegraphics[width=0.115\linewidth]{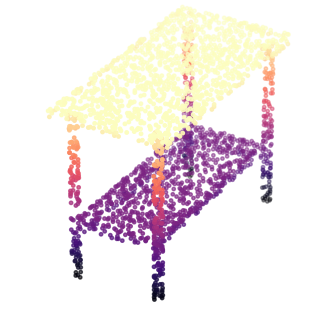}}
	\end{minipage}
	\vskip -0.3cm 
	\begin{minipage}{1.\linewidth}
		\centering
		\subfloat[F1 0.3333]
		{\includegraphics[width=0.115\linewidth]{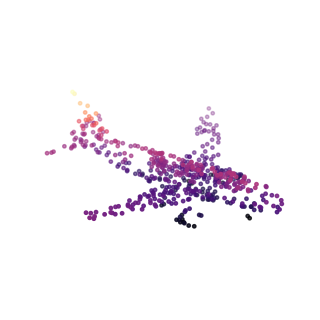}}
		\hspace{0.01cm}
		\subfloat[F1 0.9029]
		{\includegraphics[width=0.115\linewidth]{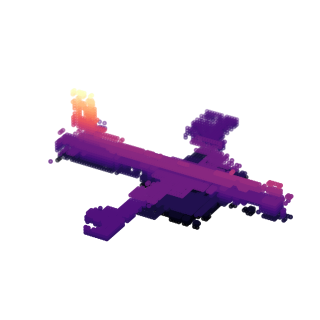}}
		\hspace{0.01cm}
		\subfloat[F1 0.9188]
		{\includegraphics[width=0.115\linewidth]{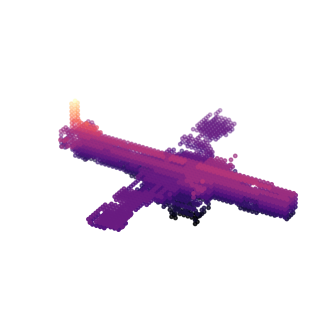}}
		\hspace{0.01cm}
		\subfloat[F1 0.9348]
		{\includegraphics[width=0.115\linewidth]{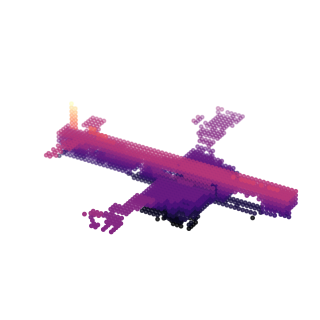}}
		\hspace{0.01cm}
		\subfloat[F1 0.9435]
		{\includegraphics[width=0.115\linewidth]{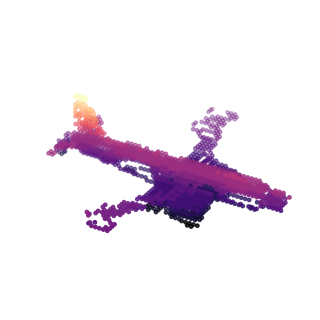}}
		\hspace{0.01cm}
		\subfloat[F1 0.9482]
		{\includegraphics[width=0.115\linewidth]{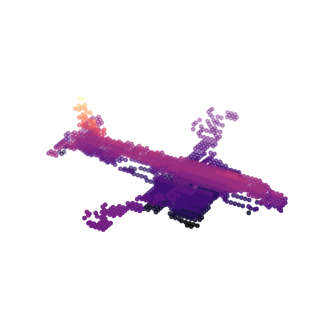}}
		\hspace{0.01cm}
		\subfloat[F1 0.9566]
		{\includegraphics[width=0.115\linewidth]{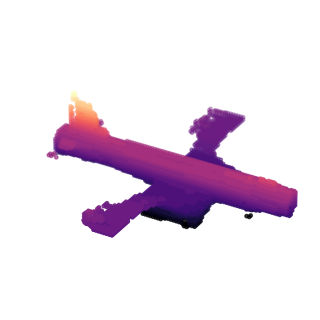}}
		\hspace{0.01cm}
		\subfloat[F1 1.0000]
		{\includegraphics[width=0.115\linewidth]{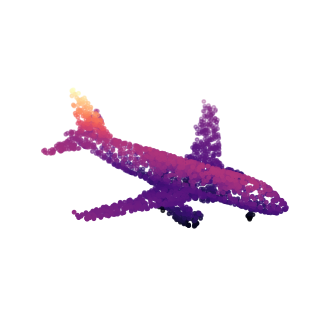}}
	\end{minipage}
	\vskip -0.3cm 
	\begin{minipage}{1.\linewidth}
		\centering
		\subfloat[F1 0.3333]
		{\includegraphics[width=0.115\linewidth]{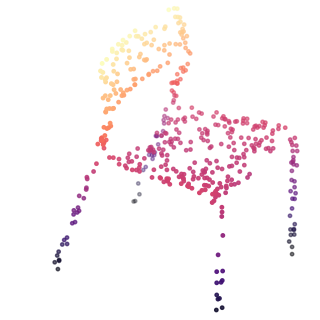}}
		\hspace{0.01cm}
		\subfloat[F1 0.9347]
		{\includegraphics[width=0.115\linewidth]{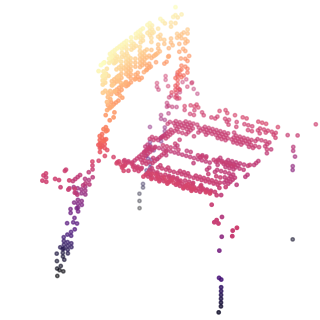}}
		\hspace{0.01cm}
		\subfloat[F1 0.9045]
		{\includegraphics[width=0.115\linewidth]{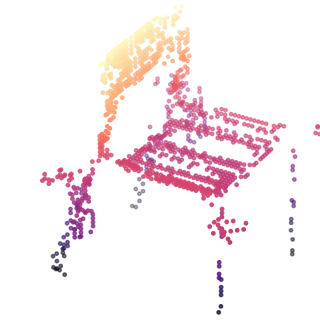}}
		\hspace{0.01cm}
		\subfloat[F1 0.8888]
		{\includegraphics[width=0.115\linewidth]{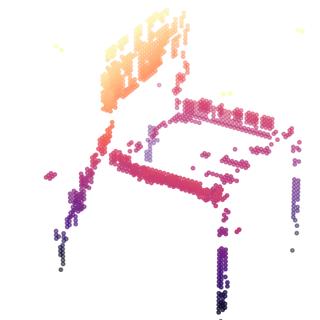}}
		\hspace{0.01cm}
		\subfloat[F1 0.9663]
		{\includegraphics[width=0.115\linewidth]{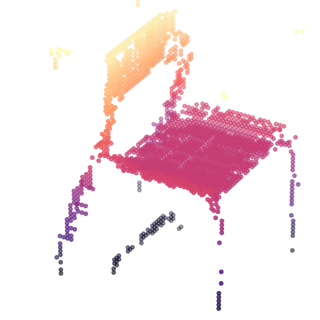}}
		\hspace{0.01cm}
		\subfloat[F1 0.9805]
		{\includegraphics[width=0.115\linewidth]{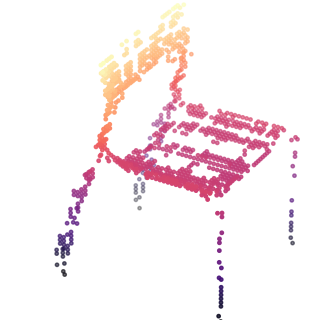}}
		\hspace{0.01cm}
		\subfloat[F1 0.9867]
		{\includegraphics[width=0.115\linewidth]{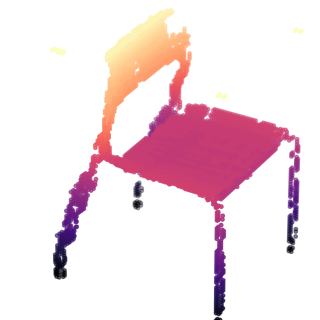}}
		\hspace{0.01cm}
		\subfloat[F1 1.0000]
		{\includegraphics[width=0.115\linewidth]{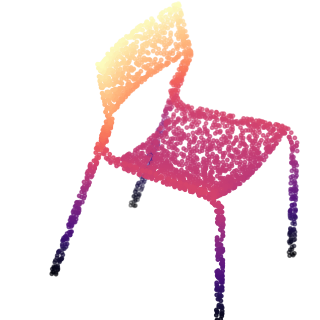}}
	\end{minipage}
	\vskip -0.3cm 
	\begin{minipage}{1.\linewidth}
		\centering
		\subfloat[\begin{tabular}{c} F1 0.3333\\ Observed\end{tabular}]
		{\includegraphics[width=0.115\linewidth]{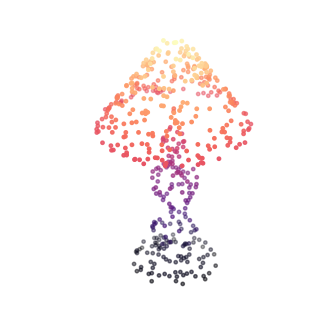}}
		\hspace{0.01cm}
		\subfloat[\begin{tabular}{c} F1 0.9138\\ SAPCU\cite{zhao2022self}\end{tabular}]
		{\includegraphics[width=0.115\linewidth]{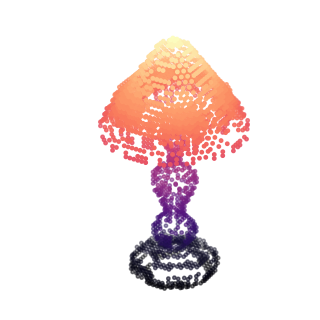}}
		\hspace{0.01cm}
		\subfloat[\begin{tabular}{c}F1 0.9475\\ NeuralTPS\cite{chen2023unsupervised}\end{tabular}]
		{\includegraphics[width=0.115\linewidth]{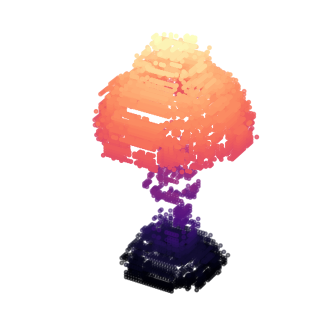}}
		\hspace{0.01cm}
		\subfloat[\begin{tabular}{c}F1 0.8883\\ NeuralPoints\cite{feng2022neural}\end{tabular}]
		{\includegraphics[width=0.115\linewidth]{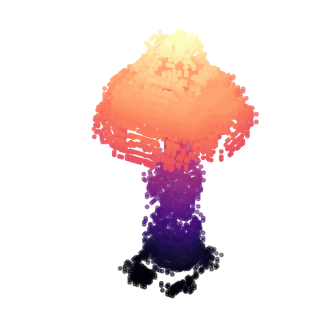}}
		\hspace{0.01cm}
		\subfloat[\begin{tabular}{c}F1 0.9696\\ Grad-PU\cite{he2023grad}\end{tabular}]
		{\includegraphics[width=0.115\linewidth]{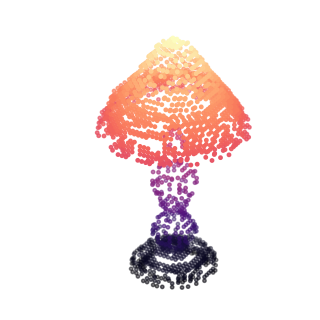}}
		\hspace{0.01cm}
		\subfloat[\begin{tabular}{c} F1 0.9860\\ LRTFR\cite{luo2023low}\end{tabular}]
		{\includegraphics[width=0.115\linewidth]{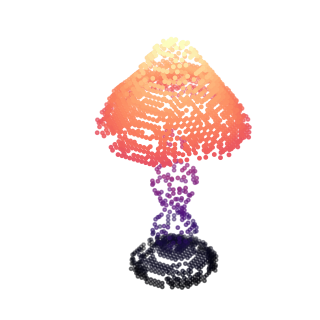}}
		\hspace{0.01cm}
		\subfloat[\begin{tabular}{c} F1 0.9913\\ CP-Pruner\end{tabular}]
		{\includegraphics[width=0.115\linewidth]{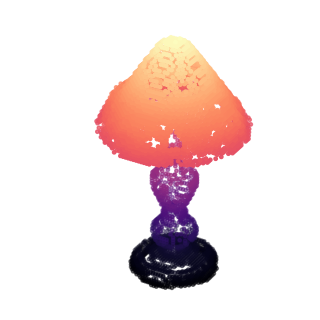}}
		\hspace{0.01cm}
		\subfloat[\begin{tabular}{c} F1 1.0000\\ Original\end{tabular}]
		{\includegraphics[width=0.115\linewidth]{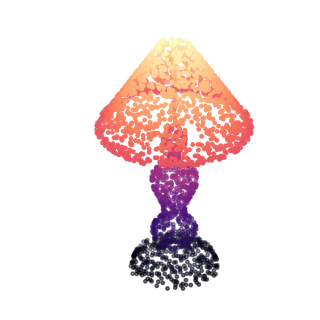}}
	\end{minipage}
	\caption{Results of point cloud upsampling by different methods on \textit{Table}, \textit{Airplane}, \textit{Chair}, and \textit{Lamp} in the ShapeNet dataset\cite{chang2015shapenet}. The number of observed points is 20\% of the original points, with fewer than 500 points in total.}
	\label{fig:Demo4UpsamplingOnShapeNet}
\end{figure*}
To demonstrate our method’s effectiveness beyond structured meshgrid data, we address the point cloud upsampling problem \cite{yu2018pu, li2019pu}. Most traditional low-rank tensor-based methods are ill-suited for this task, as they cannot represent beyond-meshgrid point clouds—by contrast, our method excels here because it learns a continuous data representation.
Given an observed sparse point cloud $\Omega = \{\mathbf{v}_i \in \mathbb{R}^3\}_{i=1}^{n}$ (where $n$ is the number of points), we use the signed distance function (SDF) \cite{park2019deepsdf} to model this continuous representation. The training loss function for the SDF is defined as:
\begin{equation}
	\label{equ:PCupsampling}
	\begin{aligned}
		\mathop{\min}_{\{\mathbf{W}^{(d)}_l\}}&\sum_{\mathbf{v}_i\in\Omega}|s(\mathbf{v}_i)|+\lambda_1\int_{\mathbb{R}^3}|\Vert\frac{\partial s(\mathbf{v}_i)}{\partial \mathbf{v}_i}\Vert^2_F-1|\dif \mathbf{v}_i\\
		&+\lambda_2\int_{\mathbb{R}^3\backslash\Omega}\exp(-|s(\mathbf{v}_i)|)\dif \mathbf{v}_i + \mathfrak{R}(\mathcal{T}).
	\end{aligned}
\end{equation}
Here, $s(\cdot): \mathbb{R}^3 \to \mathbb{R}$ denotes the SDF to be learned, and $\lambda_1$, $\lambda_2$ are trade-off parameters that balance the loss terms: The first term enforces the SDF to be zero at observed points; The second term ensures the SDF’s gradient magnitude is 1 everywhere, promoting a smooth surface; The third term encourages SDF values outside $\Omega$ to be far from zero, helping define the shape boundary. In practice, we approximate the integrals by randomly sampling a large number of spatial points.

The surface defined by $s(\mathbf{v}) = 0$ represents the underlying shape of the point cloud. For upsampling, we use evenly spaced sampling to generate dense points $\mathbf{v}$ where $|s(\mathbf{v})| < \tau_{\text{thr}}$, with $\tau_{\text{thr}}$ as a predefined threshold, these points form the desired high-resolution point cloud.
This approach leverages the learned SDF’s continuity to upsample point clouds effectively while preserving the original shape’s structural integrity. We normalize the point cloud coordinates and set $\tau_{\text{thr}} = 0.05$, a value chosen to ensure the recovered dense point cloud contains at least $10^5$ points.

We conducted experiments on multiple datasets: Table, Airplane, Chair, and Lamp from the ShapeNet benchmark \cite{chang2015shapenet}, the Stanford Bunny\footnote{\url{https://graphics.stanford.edu/data/3Dscanrep/}}, and three hand-crafted shapes (Doughnut, Sphere, Heart). For each dataset, we downsampled the original point cloud to an observed sparse set (fewer than 500 points) using random sampling.
We compared our method with five deep learning-based baselines: SAPCU \cite{zhao2022self}, NeuralTPS \cite{chen2023unsupervised}, NeuralPoints \cite{feng2022neural}, Grad-PU \cite{he2023grad}, and LRTFR \cite{luo2023low}. Results for point cloud upsampling are shown in Table \ref{tab:comparePoint} and Fig.~\ref{fig:Demo4UpsamplingOnShapeNet}, which demonstrate that our method consistently generates denser point clouds, achieves significant improvements in both quantitative metrics and qualitative evaluations, and exhibits strong generalization across diverse datasets. Additional visualizations are provided in the supplementary material.
This superior performance stems from three key factors: first, our unsupervised approach requires no training dataset, ensuring greater versatility; second, the low-rank regularization in Eq.~\eqref{equ:SumofSpNorm} is explicitly designed for sparse CP decomposition, boosting representation efficiency; third, the Jacobian-based smoothness regularization in Eq.~\eqref{equ:SmoothnessReg} is naturally compatible with continuous data, enabling better generalization across datasets.

\section{DISCUSSIONS}
\begin{figure*}[t]
	\centering
	\includegraphics[width=0.19\linewidth]{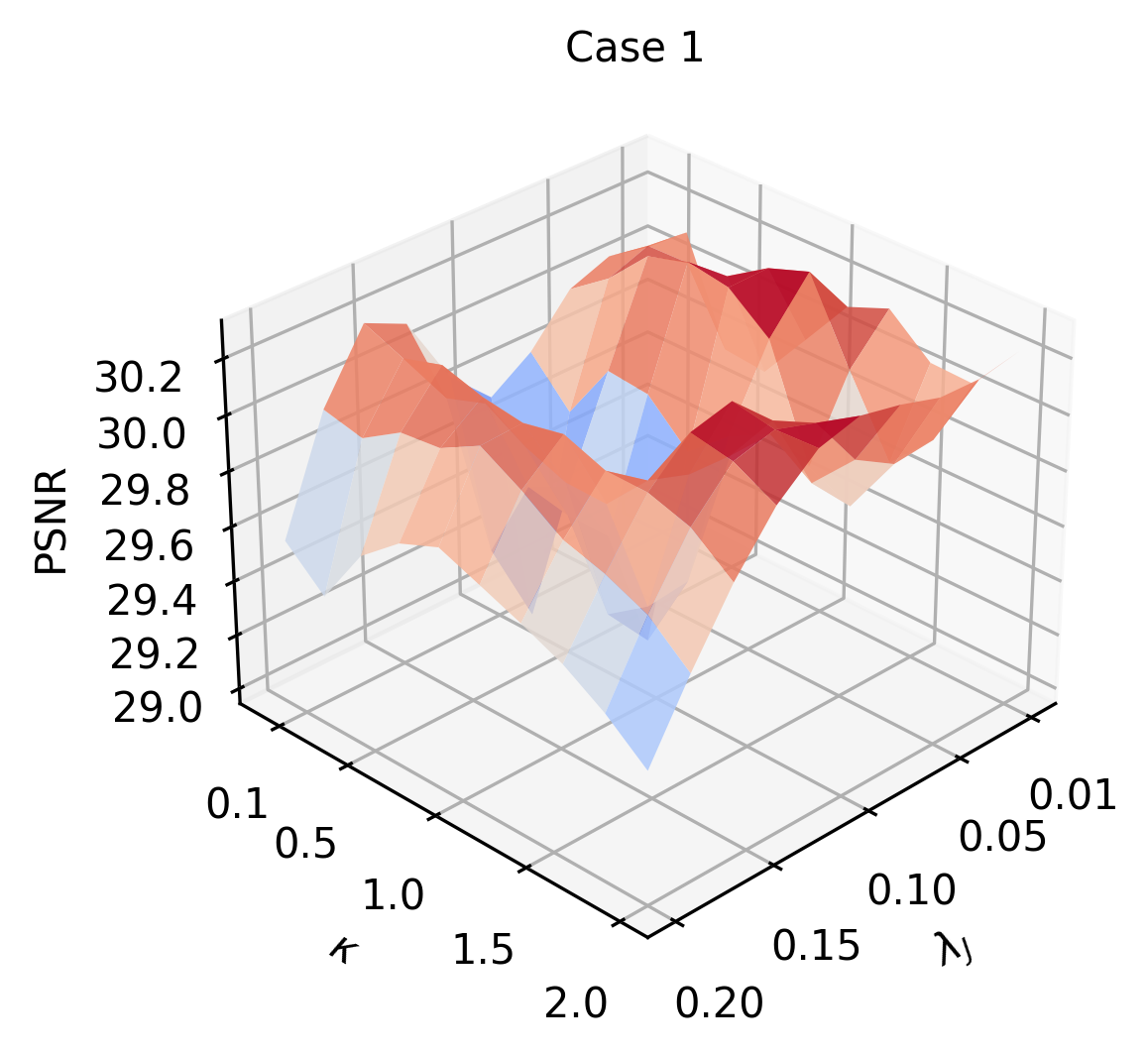}
	\includegraphics[width=0.19\linewidth]{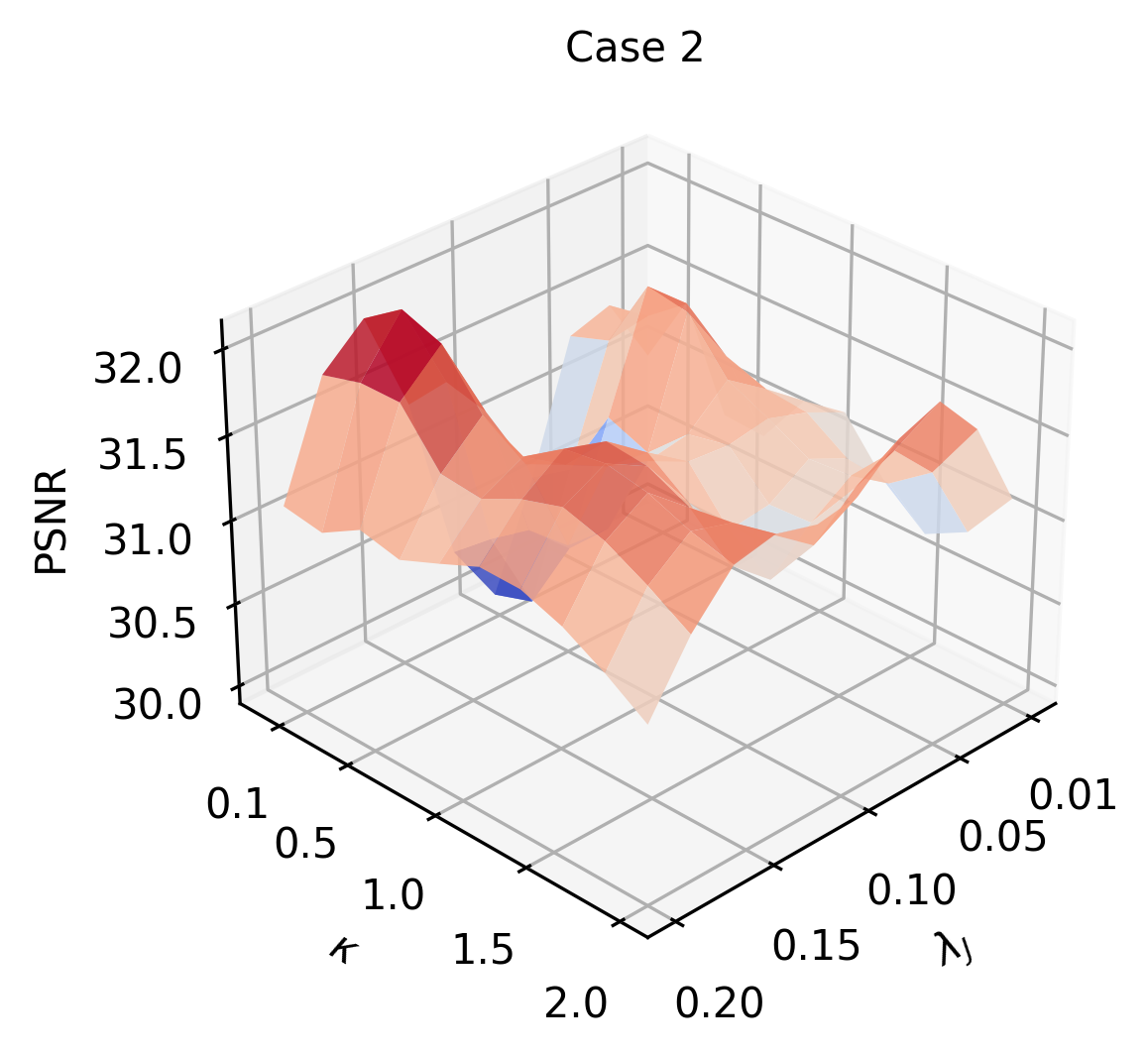}
	\includegraphics[width=0.19\linewidth]{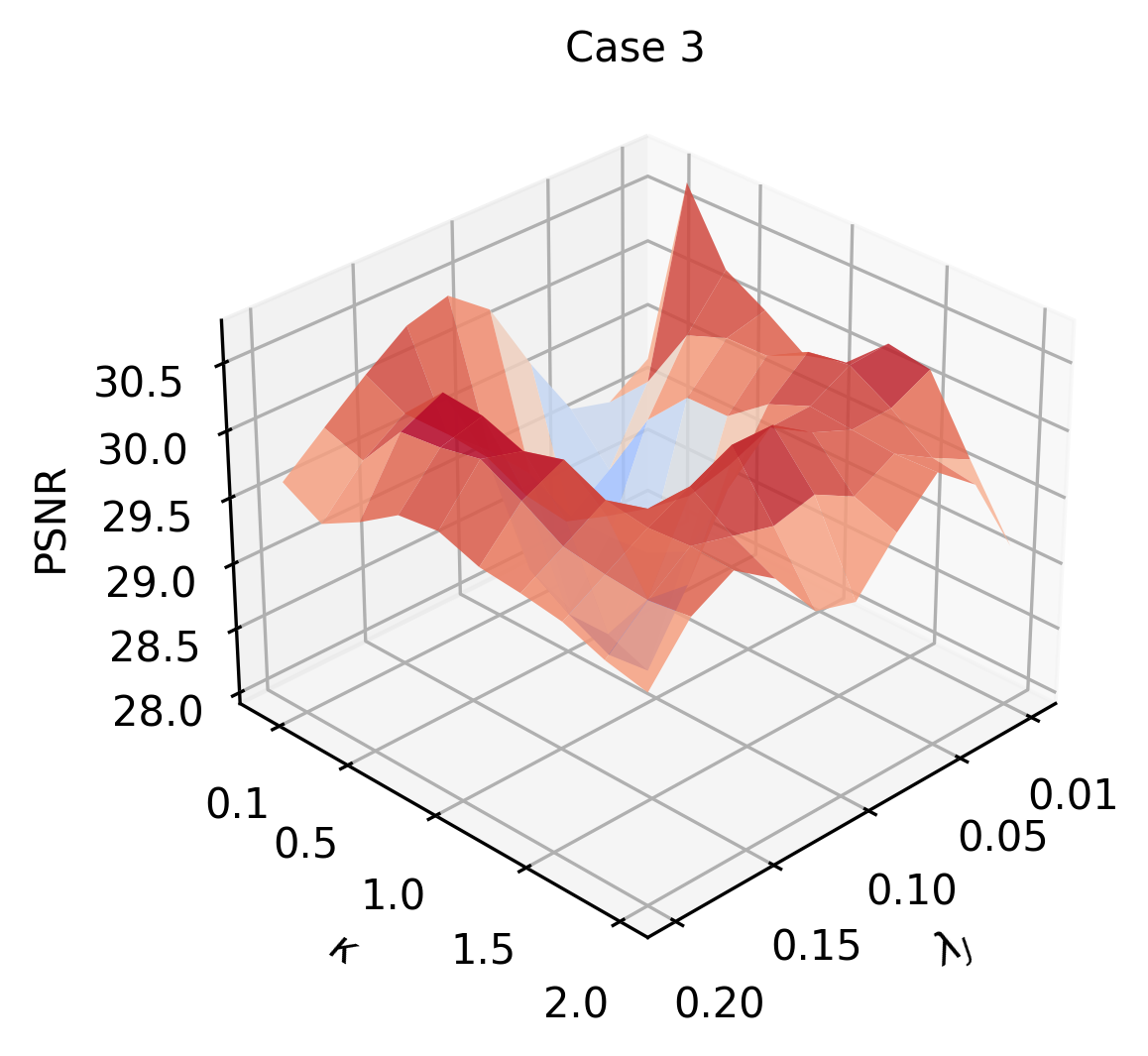}
	\includegraphics[width=0.19\linewidth]{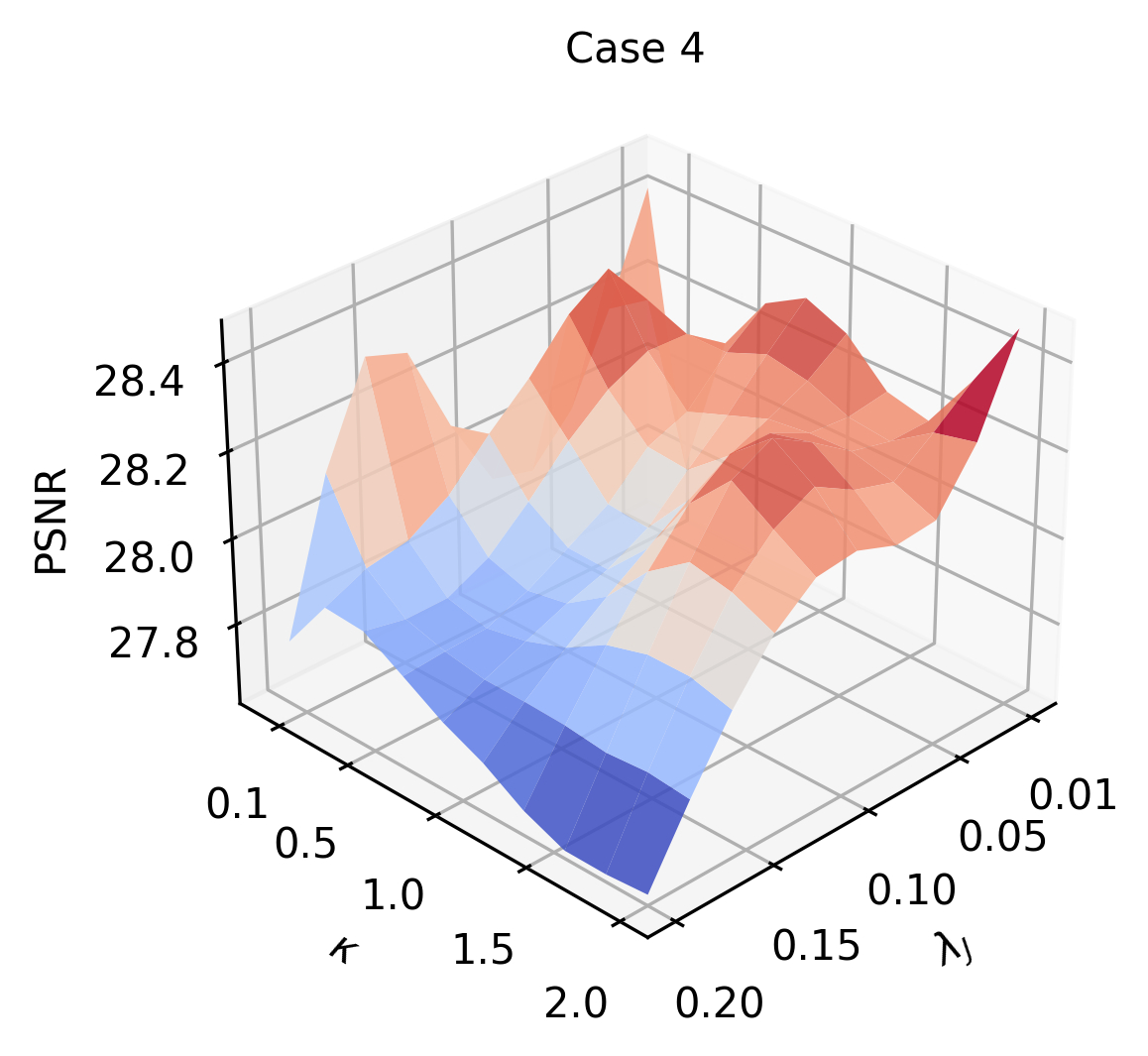}
	\includegraphics[width=0.19\linewidth]{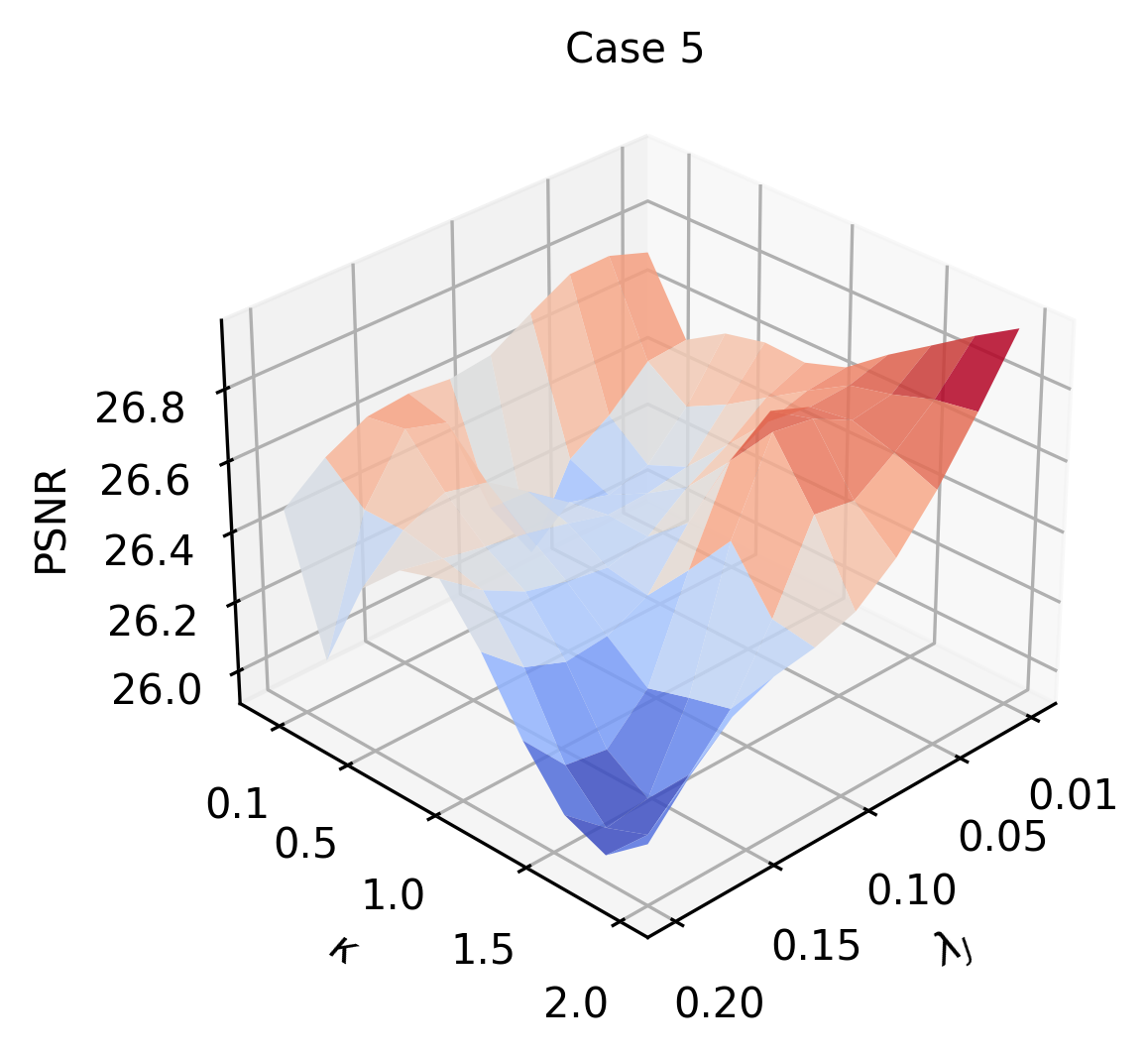}
	\caption{Sensitivity of $\kappa$ and $\lambda_{J}$ on \textit{PaviaU} denoising task.}
	\label{fig:kappaAblation}
\end{figure*}
\subsection{Influences of Schatten-p quasi-norm}
\begin{figure}[t]
	\centering
	\begin{minipage}{1.\linewidth}
		\centering
		\subfloat[]
		{\includegraphics[width=0.5\linewidth]{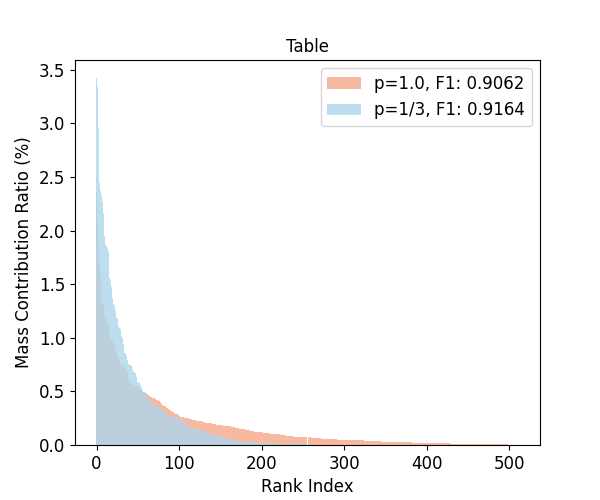}}
		\subfloat[]
		{\includegraphics[width=0.5\linewidth]{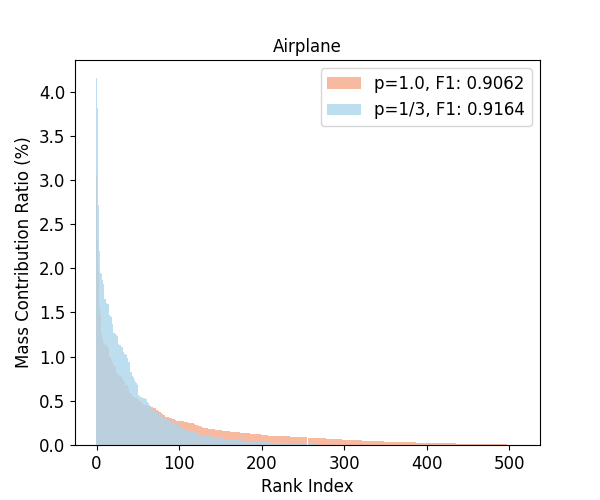}}
	\end{minipage}
	\vskip -0.8cm 
	\begin{minipage}{1.\linewidth}
		\centering
		\subfloat[]
		{\includegraphics[width=0.5\linewidth]{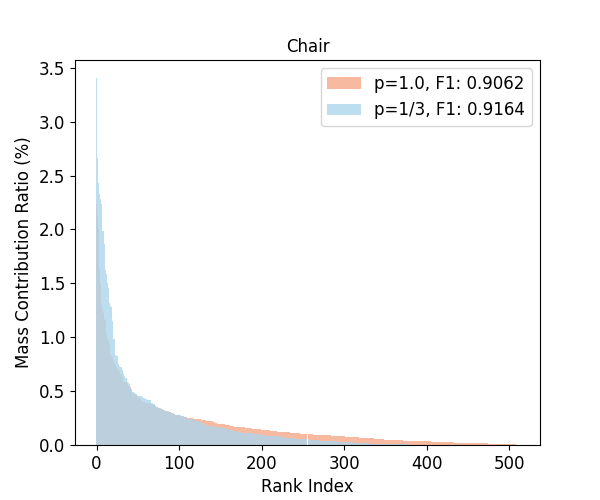}}
		\subfloat[]
		{\includegraphics[width=0.5\linewidth]{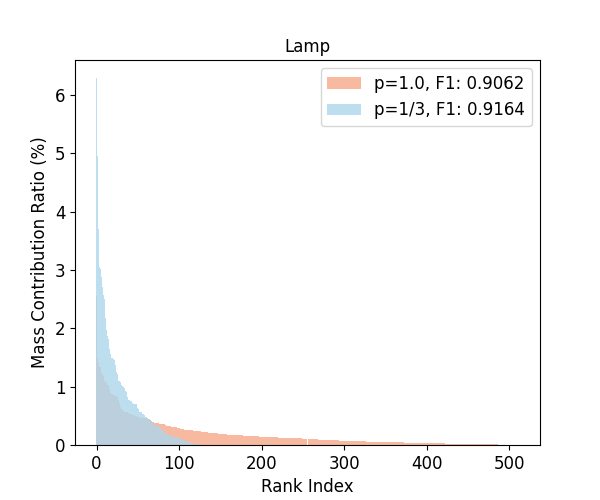}}
	\end{minipage}
	\caption{Comparison of the percentage of total mass accounted for by a single Rank-1 Tensor in the CP decomposition of upsampled point clouds, across different values of $p$}
	\label{fig:CPWeights}
\end{figure}
The variational Schatten-$p$ quasi-norm is a key regularization term in our method. For $p<1$, the quasi-norm promotes sparsity more aggressively than the nuclear norm, which explains the improved performance. As illustrated in Fig.~\ref{fig:CPWeights}, the CP weights from upsampled points are clearly sparse—only a few components carry most of the weight—demonstrating that when the CP-rank $R$ is overspecified, the regularization automatically prunes redundant components.
We further evaluated the effect of different $p$ values across three low-rank recovery tasks, as shown in Fig.~\ref{fig:ablation4pR}, performance varies significantly with $p$, and $p<1$ consistently outperforms $p=1$. This property is particularly beneficial when the underlying data has an approximately low-CP-rank structure dominated by a few components.

Notably, the non-convexity of the variational Schatten-$p$ quasi-norm is long viewed as a challenge in traditional optimization, but in our model, it can be effectively addressed by deep learning optimizers. Additionally, adjusting $p$ allows flexible and automatic control of sparsity, which explains why our proposed method achieves strong performance.
From these results, we recommend using smaller values of $p$ ($p \leq 10^{-1}$) to attain optimal performance across the tested tasks. This further validates that the variational Schatten-$p$ quasi-norm effectively promotes low-CP-rank structures.

\subsection{Influences of the choice of predefined rank $R$}
\begin{figure}[t]
	\centering
	\begin{minipage}{1.\linewidth}
		\centering
		\subfloat[(a) $p$ in inpainting]
		{\includegraphics[width=0.49\linewidth]{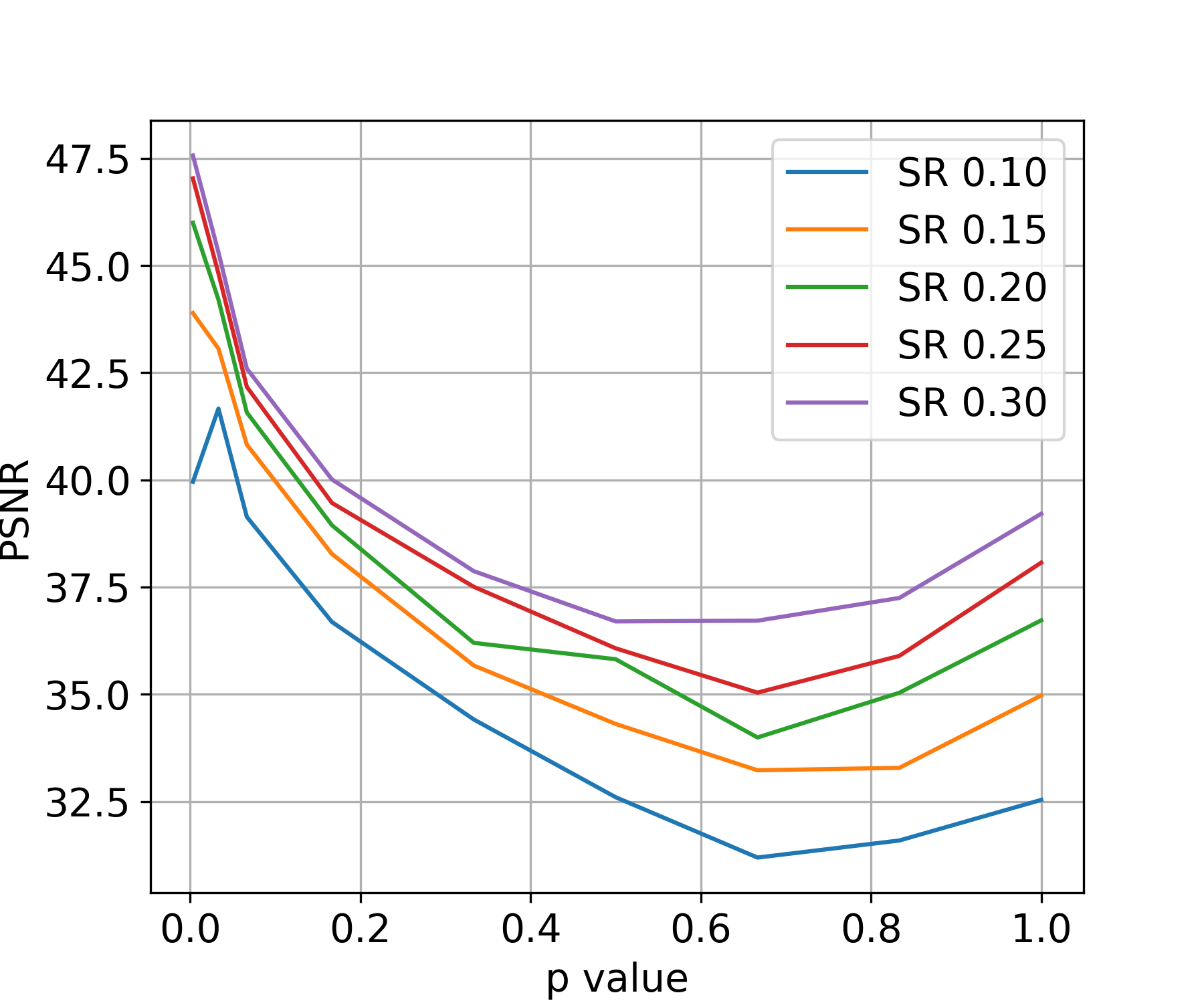}}
		\subfloat[(b) $R$ in inpainting]
		{\includegraphics[width=0.49\linewidth]{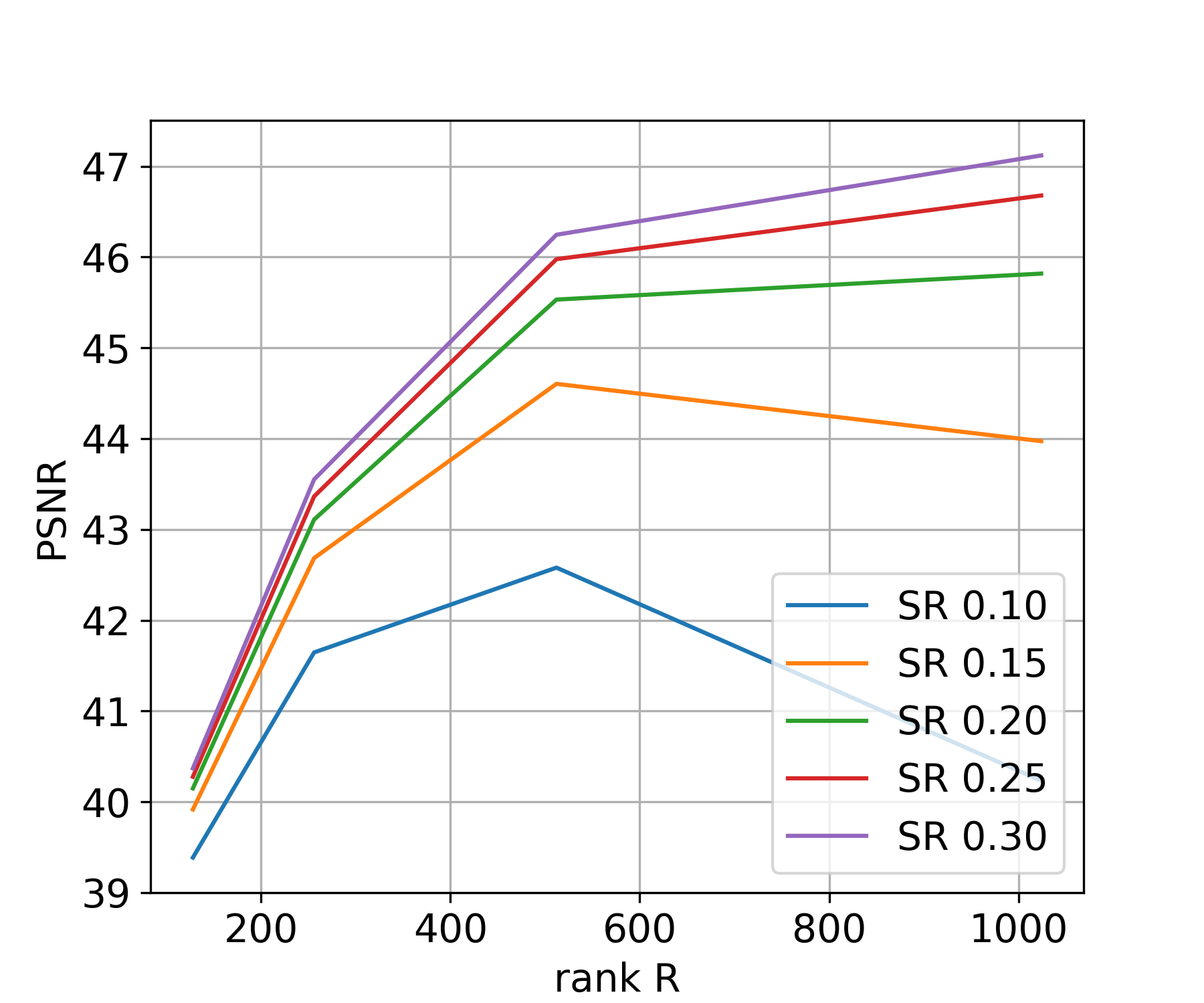}}
	\end{minipage}
	\vskip -0.3cm 
	\begin{minipage}{1.\linewidth}
		\centering
		\subfloat[(c) $p$ in denoising]
		{\includegraphics[width=0.49\linewidth]{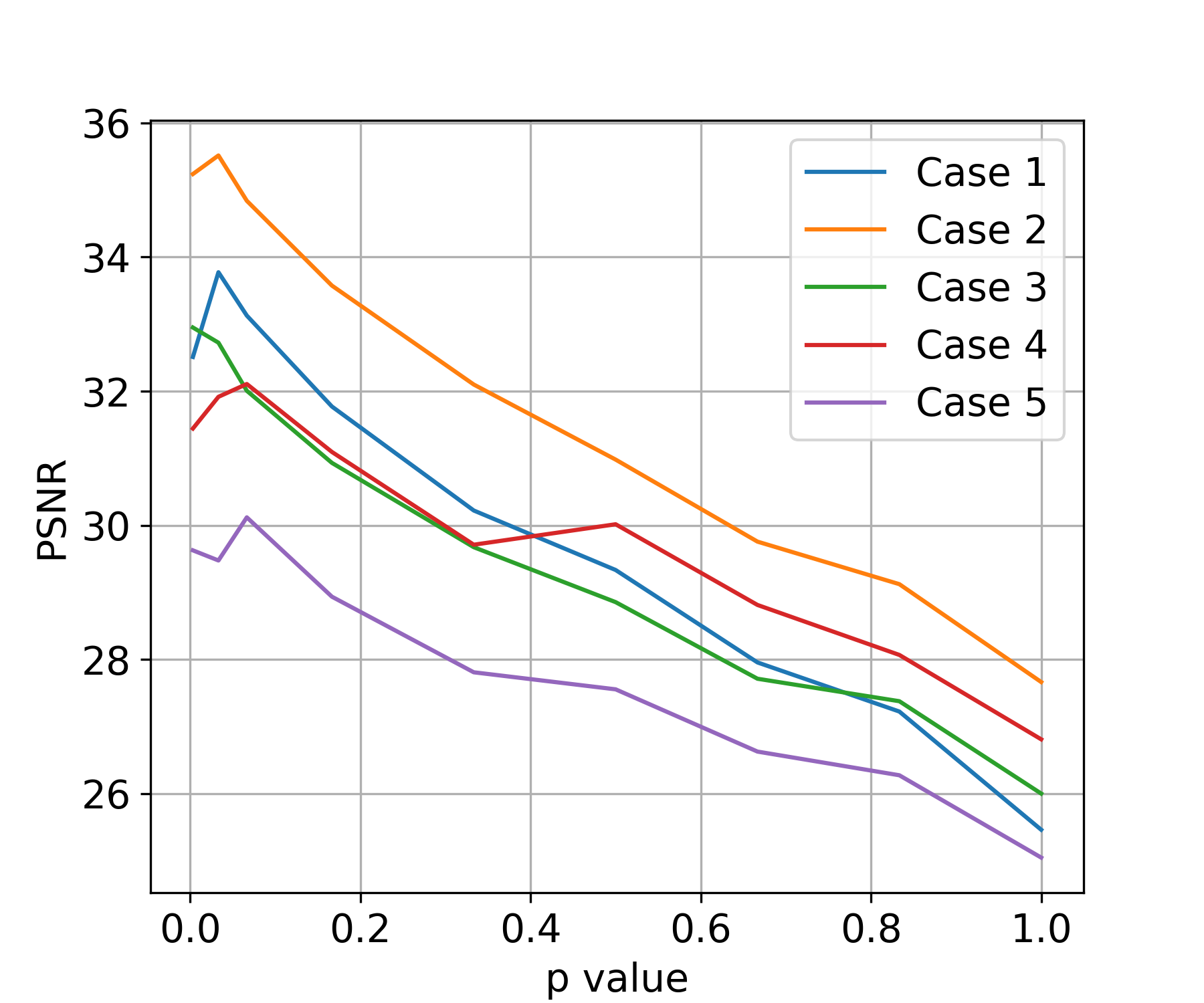}}
		\subfloat[(d) $R$ in denoising]
		{\includegraphics[width=0.49\linewidth]{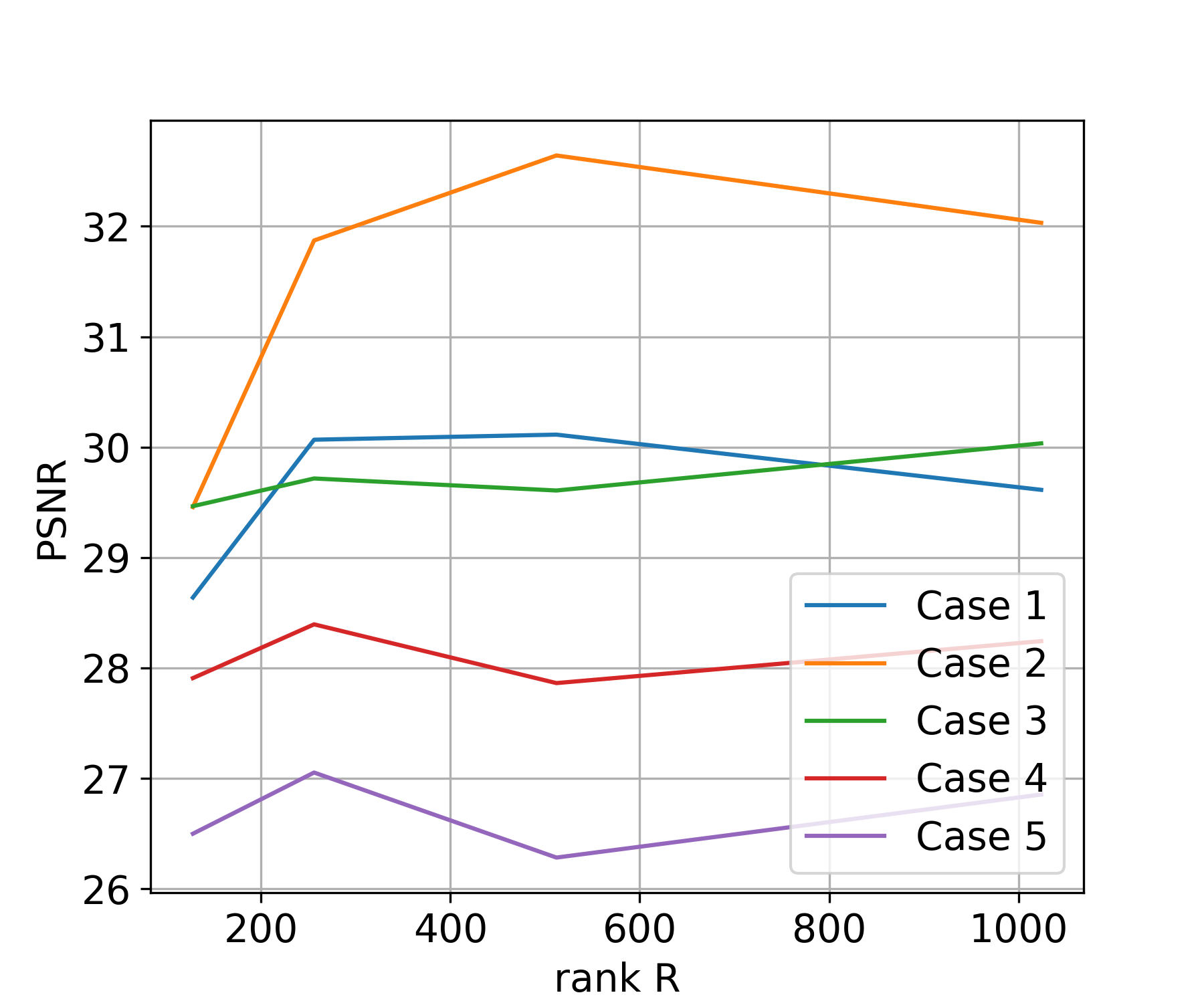}}
	\end{minipage}
	\vskip -0.3cm 
	\begin{minipage}{1.\linewidth}
		\centering
		\subfloat[(e) $p$ in upsampling]
		{\includegraphics[width=0.49\linewidth]{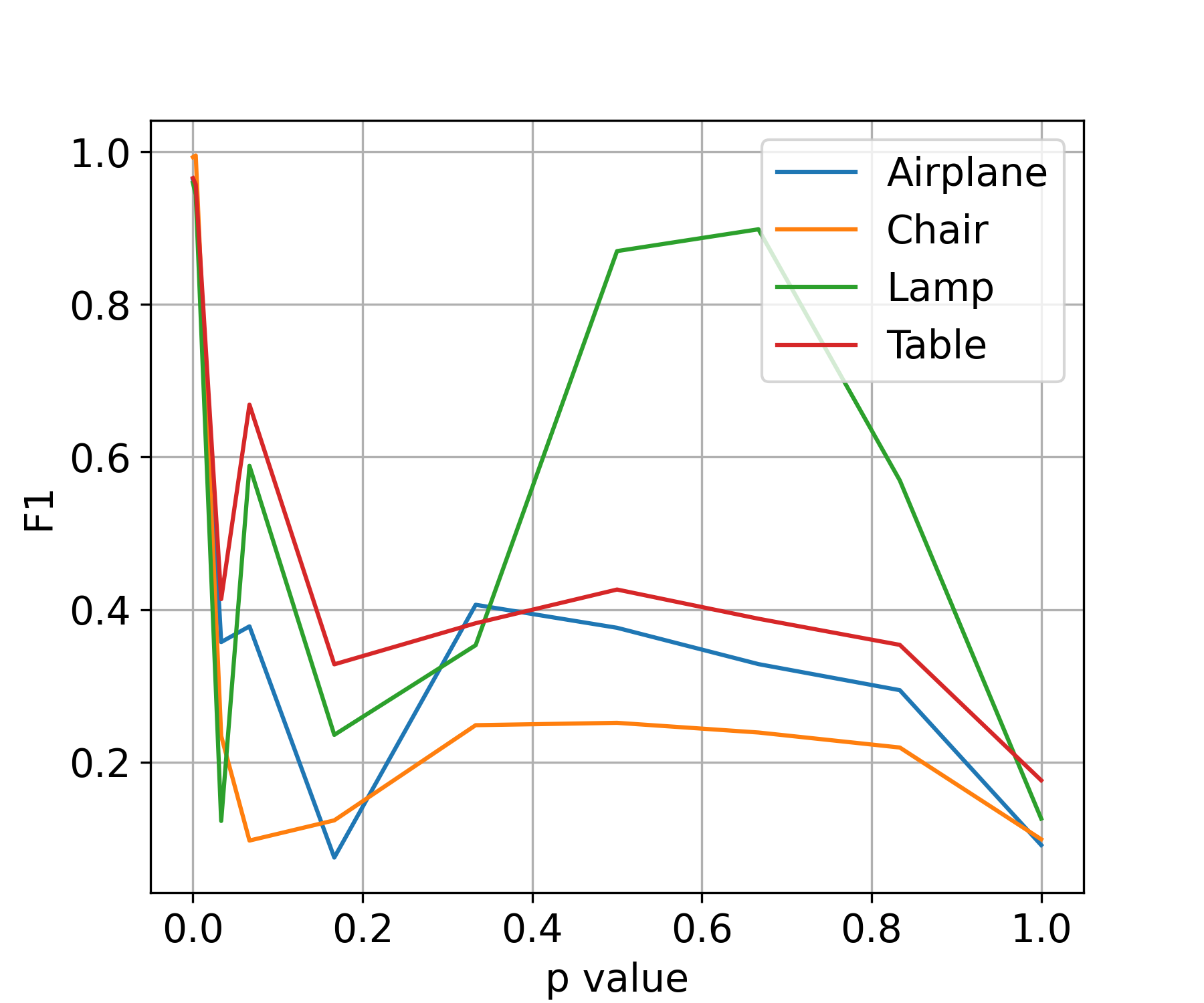}}
		\subfloat[(f) $R$ in upsampling]
		{\includegraphics[width=0.49\linewidth]{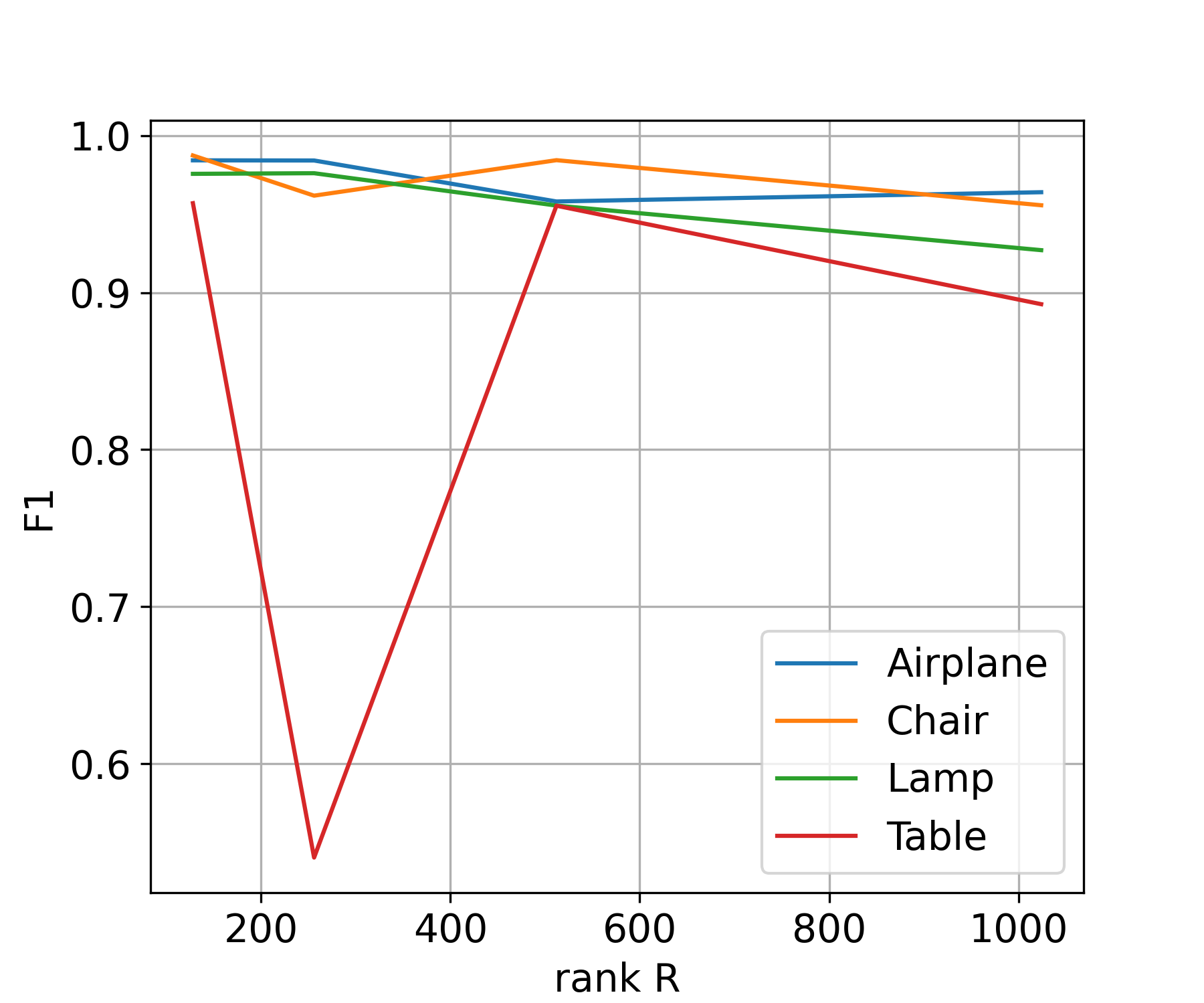}}
	\end{minipage}
	\caption{Ablation study of $p$ value in variational Schatten-p quasi-norm and the predefined CP rank $R$. The inpainting, denoising and upsampling experiment on MSIs, \textit{PaviaU} and ShapeNet.}
	\label{fig:ablation4pR}
\end{figure}
Selecting an appropriate rank upper bound $R$ is critical across all matrix and tensor decomposition methods. Therefore, a very important issue is whether a tensor recovery method can work on all predefined values of $R$ without performance collapse.
The robustness test of our method with respect to $R$ is shown in Fig.~\ref{fig:ablation4pR}. Our method is not highly sensitive to $R$ across tasks. It is also relatively robust to this hyperparameter and achieves satisfactory performance over a wide range of values in most cases.
Fig.~\ref{fig:ablation4pR} also shows that a larger $R$ provides more stable performance guarantees. However, it also increases the number of network parameters. Thus, a higher rank can enhance stability, but it is essential to balance stability with computational efficiency.

\subsection{Influences of Smooth Regularization}
\begin{table}[t]
	\centering
	\caption{Comparison between TV and Jacobian-based regularization for smoothness on \textit{PaviaU}.}
	\resizebox{1.0\linewidth}{!}{
	\begin{tabular}{ccccccc}
		\toprule
		\multicolumn{2}{c}{Noise case}                      
		&Case 1 &Case 2 &Case 3 &Case 4 &Case 5 \\ 
		\midrule
		\multirow{3}{*}{w/o Smooth Reg} & \multicolumn{1}{c|}{PSNR} 
		&28.28 &29.83 &27.44 &28.77 &28.52 \\
		&\multicolumn{1}{c|}{SSIM} 
		&0.753 &0.813 &0.761 &0.841 &0.818 \\
		&\multicolumn{1}{c|}{NRMSE} 
		&0.188 &0.157 &0.207 &0.179 &0.184 \\ 
		\midrule
		\multirow{3}{*}{TV Reg} & \multicolumn{1}{c|}{PSNR} 
		&29.09 &30.73 &27.72 &28.64 &29.13 \\
		&\multicolumn{1}{c|}{SSIM} 
		&0.797 &0.846 &0.773 &0.853 &0.838 \\
		&\multicolumn{1}{c|}{NRMSE} 
		&0.171 &0.142 &0.201 &0.181 &0.171 \\ 
		\midrule
		\multirow{3}{*}{\begin{tabular}{c}Jacobian-based Reg\\$(\lambda_{J}=0.01, \kappa=1.0)$\end{tabular}} & \multicolumn{1}{c|}{PSNR} 
		&\textbf{30.45} &\textbf{33.69} &\textbf{31.98} &\textbf{29.27} &\textbf{29.15} \\
		&\multicolumn{1}{c|}{SSIM} 
		&\textbf{0.827} &\textbf{0.905} &\textbf{0.877} &\textbf{0.846} &\textbf{0.843} \\
		&\multicolumn{1}{c|}{NRMSE} 
		&\textbf{0.147} &\textbf{0.101} &\textbf{0.123} &\textbf{0.168} &\textbf{0.170} \\
		\bottomrule
	\end{tabular}}
\label{tab:TVvsSmoothReg}
\end{table}
We evaluated the influence of Jacobian-based regularization on the denoising performance of multispectral images using the Pavia University dataset, as summarized in Table~\ref{tab:TVvsSmoothReg}. 
As noted in DeepTensor \cite{saragadam2024deeptensor}, low-dimensional tensor decomposition is an effective approach for tensor principal component analysis, particularly useful in handling gross outliers such as salt-and-pepper noise. Therefore, even without additional smoothing constraints, low-rank tensor functions still exhibit notable denoising capabilities, as demonstrated in the first row of Table~\ref{tab:TVvsSmoothReg}.
Comparison between the second row and the third row show that our Jacobian-based smoothness regularization significantly outperforms traditional TV regularization. 

Classical total variation regularization is grid dependency and widely used for image denoising. However, it cannot constrain the smoothness of non-grid data such as point clouds. In contrast, our proposed explicit Jacobian-based smoothness regularization can be readily extended to various tensor representations, including non-grid data.
In addition to benefiting image denoising tasks, our Jacobian-based smoothness regularization can control the sampling density of point clouds by adjusting its parameters. As shown in Fig.~\ref{fig:Demo4UpsamplingOnShapeNet}, the point clouds generated by our method are denser than those produced by other methods. This capability is not achievable with TV regularization.

We performed sensitivity tests on the hyperparameters $\kappa$ and $\lambda_J$ in Eq.~\eqref{equ:regularizations}, as shown in Fig.~\ref{fig:kappaAblation}. The parameters were adjusted within the ranges $\kappa \in [0.01, 2.0]$ and $\lambda_J \in [0.01, 0.2]$.
Under five different noise cases, our denoising performance remains stable with a variation of less than 5\%. No performance collapse is observed, indicating strong robustness.

\section{Conclusion} 
We generalize CP decomposition to the functional domain to model both on and beyond-grid tensor data. A variational Schatten-$p$ quasi-norm is introduced to induce a sparser CP decomposition by automatically pruning redundant components during deep learning optimization. We prove that the variational Schatten-$p$ quasi-norm is an upper bound for any of its matricizations. We also propose a Jacobian-based smoothness regularization that can be applied to any differentiable tensor function and can serve as a partial substitute for total variation loss.
By leveraging the inherent properties of CP decomposition and advanced regularization techniques within a deep learning framework, our method provides a robust framework for low-CP-rank tensor representation and achieves superior performance in various real-world applications.


\bibliographystyle{IEEEtran}
\bibliography{reference} 

\end{document}